\documentclass[11pt]{article}

\usepackage[T1]{fontenc}
\usepackage[utf8]{inputenc}
\usepackage{lmodern}
\usepackage[margin=1in]{geometry}
\usepackage{amsmath,amssymb,amsthm}
\usepackage{thm-restate}
\usepackage{mdframed}
\usepackage{mathtools}
\usepackage{stmaryrd}
\usepackage{microtype}
\usepackage{graphicx}
\usepackage{booktabs}
\usepackage{enumitem}
\usepackage{needspace}
\usepackage{tikz}
\usetikzlibrary{arrows.meta,calc,decorations.pathreplacing}
\usepackage{amartya_ltx}
\usepackage{algorithm}
\usepackage{algpseudocode}
\usepackage[font=small,skip=6pt]{caption}
\usepackage{xurl}
\usepackage[hidelinks,hypertexnames=false]{hyperref}
\usepackage[nameinlink,noabbrev]{cleveref}

\theoremstyle{plain}

\newtheorem{lemma}{Lemma}
\newtheorem{informaltheorem}{Theorem}

\definecolor{theoremfill}{RGB}{243,247,250}
\definecolor{theoremborder}{RGB}{160,180,196}
\surroundwithmdframed[
  backgroundcolor=theoremfill,linecolor=theoremborder,linewidth=0.8pt,
  leftmargin=-8pt,rightmargin=-8pt,
  innerleftmargin=7.5pt,innerrightmargin=7.5pt,
  innertopmargin=0pt,innerbottommargin=6pt,
  skipabove=3pt,skipbelow=3pt,nobreak=true]{informaltheorem}
\newtheorem{proposition}{Proposition}
\newtheorem{fact}{Fact}
\theoremstyle{definition}
\newtheorem*{question}{Question}
\surroundwithmdframed[
  backgroundcolor=theoremfill,linecolor=theoremborder,linewidth=0.8pt,
  leftmargin=-8pt,rightmargin=-8pt,
  innerleftmargin=7.5pt,innerrightmargin=7.5pt,
  innertopmargin=0pt,innerbottommargin=6pt,
  skipabove=3pt,skipbelow=3pt,nobreak=true]{question}
\newtheorem{definition}{Definition}
\newmdenv[linewidth=0.5pt,innerleftmargin=8pt,innerrightmargin=8pt,
  innertopmargin=0pt,innerbottommargin=4pt,
  skipabove=8pt,skipbelow=8pt,nobreak=true]{gamebox}
\theoremstyle{plain}
\DeclareMathOperator{\Var}{Var}
\DeclareMathOperator{\Ldim}{Ldim}
\newcommand{\Point}{\mathrm{POINT}_3}
\allowdisplaybreaks[1]

\hypersetup{
  pdftitle={Private online learning and prediction for Littlestone classes},
  pdfauthor={Amartya Sanyal}
}
\title{Private online learning and prediction for Littlestone classes}
\author{Amartya Sanyal\\
  \small Department of Computer Science\\
  University of Copenhagen\\
  \texttt{amsa@di.ku.dk}}
\date{}

\begin{document}
\maketitle

\begin{abstract}
We study differentially private online learning and jointly private online
prediction against oblivious realisable adversaries. Online learning
releases a hypothesis before observing each example, whereas online
prediction releases only a prediction for the current instance. For every
class of finite Littlestone dimension $d\ge2$, every
$(\epsilon,\delta)$-private online learner incurs
$\Omega(\min\{(d/\epsilon)(\log T)^{2/3},d/\delta\})$ expected mistakes
on some deterministic realisable stream of length $T$, provided
$0<\epsilon\le1$, $0\le\delta\le\min\{\epsilon^2,1/4\}$, and
$T\ge(d/\epsilon)^2$. We also prove an
$\Omega(\min\{(\log(2T))^{2/3},1/\delta\})$ lower bound for point
functions on a three-element domain at $\epsilon=1/100$ and
$0\le\delta\le1/4$. For private
prediction, every nonempty class of finite Littlestone dimension $d$
admits an $(\epsilon,\delta)$-jointly private predictor with at most
$2^{2^{c(d+1)^2}}\epsilon^{-2}\log^2(2/(\epsilon\delta))$ expected
mistakes, independently of $T$, for $0<\epsilon\le1$, $0<\delta<1/10$,
and an absolute constant $c>0$. Thus, for fixed $d\ge2$ and $\epsilon$,
at $\delta=\Theta(1/\log T)$, private learning requires
$\Omega((\log T)^{2/3})$ expected mistakes, whereas private prediction
admits $O((\log\log T)^2)$. The same separation holds for point
functions on three elements at $\epsilon=1/100$.

\end{abstract}

\begingroup
\section{Introduction}
\label{sec:introduction}

In realisable online learning, given a hypothesis class \(\cH\), a
learner predicts the labels of a sequence of examples consistent with
one unknown \(h\in\cH\). It is well known that if the class \(\cH\)
has Littlestone dimension $\Ldim\br{\cH}=d$, then the~\emph{Standard
Optimal Algorithm} makes at most $d$ mistakes on any stream,
independently of the horizon $T$~\cite{Littlestone88}. Additionally,
this bound is also known to be tight; thus Littlestone dimension
tightly characterises the mistake bound for realisable online
learning.

\medskip
\noindent\textbf{Private online learning} Under differential privacy,
there exist algorithms with mistake bounds logarithmic in $T$ for
fixed privacy parameters~\cite{GL21,Lyu26}.  Whether this dependence
on the horizon is necessary was posed as an open question
in~\cite{SR22}.  Subsequent lower bounds have shown that this
logarithmic dependence may be necessary, but only for certain
algorithms~\cite{DSS24}, hypothesis classes~\cite{CLNSS24}, or ranges
of $\delta$,~($\delta\le\epsilon/T$ in~\cite{LWY24}). This raises the
following question:\looseness=-1

\begin{question}
How many mistakes must an arbitrary private online learner make on a
fixed hypothesis class for every value of $\delta$?
\end{question}

We make progress on this question in the range of $\delta$ left open
by prior work, up to order $1/\log T$. We prove a lower bound
polynomial in $\log T$, first for point functions on three elements
and then for every class of finite Littlestone dimension at least two.

\noindent Our argument starts with a simpler class, point functions on a
three-element domain:
\[
  \Point=\bc{f_1,f_2,f_3},
  \qquad f_j:\bc{1,2,3}\to\bc{0,1},
  \qquad f_j\br{x}=\mathbf{1}\bc{x=j}.
\]
This class has Littlestone dimension one and a single mistake suffices
without privacy. The main point is that $f_1$ and $f_2$ agree at point
$3$ but disagree at point $1$; thus one mistake at \(1\) suffices. We
first prove a lower bound for this problem with online DP
in~\Cref{thm:informal-point} and then extend this to every class of
Littlestone dimension at least two in~\Cref{thm:informal-dimension}.

\begin{informaltheorem}
\label{thm:informal-point}
For $\epsilon_0=1/100$ and $0\le\delta\le1/4$, every
$\br{\epsilon_0,\delta}$-differentially private online learner of
$\Point$ has a realisable stream of length $T$ on which
\[
  \bE\bs{M_T}=\Om{\min\bc{\br{\log\br{2T}}^{2/3},\frac1\delta}}.
\]
\end{informaltheorem}

\begin{table}[!t]
\centering
\begingroup
\small
\setlength{\tabcolsep}{4pt}
\renewcommand{\arraystretch}{1.15}
\begin{tabular}{@{}p{0.23\linewidth}p{0.25\linewidth}p{0.25\linewidth}p{0.20\linewidth}@{}}
\toprule
Result & Point domain & DP Learner & Lower bound \\
\midrule
\cite{DSS24}
  & $\bc{1,2,3}$
  & Concentrated
  & $\Om{\log T}$ \\
\addlinespace[3pt]
\cite{CLNSS24}
  & $\bc{0,\ldots,T}$
  & All
  & $\Om{\log T}$ \\
\addlinespace[3pt]
\cite{LWY24}
  & $\bc{1,2,3}$
  & All
  & $\Om{\log\log T}$ \\
\addlinespace[3pt]
\Cref{thm:informal-point}
  & $\bc{1,2,3}$
  & All
  & $\Om{\br{\log T}^{2/3}}$ \\
\bottomrule
\end{tabular}
\caption{Lower bounds for expected total mistakes of private online
learners on \(\Point\) at $\delta=1/(\log T)$. \cite{DSS24,LWY24}'s
results also extend to other Littlestone classes.  \cite{CLNSS24}
considers point functions on $\bc{0,\ldots,T}$, a domain that grows
with $T$.  The~\cite{LWY24} entry follows by setting $T'=\floor{\epsilon\log T}$, applying their theorem to the first $T'$ rounds, and filling the
remaining rounds with dummy points.}
\label{tab:point-comparison}
\endgroup
\end{table}

At $\delta=\Theta\br{1/\log T}$,~\Cref{thm:informal-point} improves
the $\Om{\log\log T}$~\cite[see~\Cref{tab:point-comparison}
caption]{LWY24} to $\Om{\br{\log T}^{2/3}}$ for arbitrary private
learners on $\Point$.~\Cref{tab:point-comparison} compares our bound
to existing bounds for this range of \(\delta\). As shown
in~\Cref{fig:delta-phase}, a gap remains to the $\bigO{\log T}$ upper
bound~\cite{GL21,LWY24} for this range of \(\delta\). For polynomially
small $\delta$, the $\Om{\log T}$ lower bound of~\cite{LWY24} is
stronger and tight.  When $\delta\ge\br{\log
T}^{-2/3}$,~\Cref{thm:informal-point} matches the $\bigO{1/\delta}$
upper bound using the name-and-shame variant of SOA~\cite{DSS24}.

\suppressfloats[t]
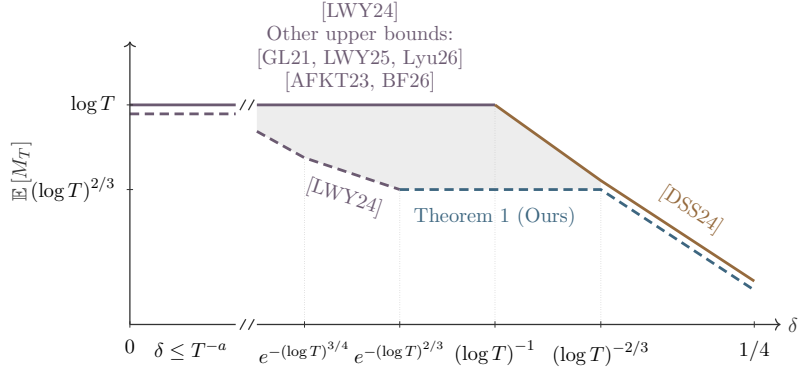
\begin{figure}[t]
\centering
\begingroup
\definecolor{phaseBlue}{HTML}{3D6680}
\definecolor{phasePurple}{HTML}{6B5A71}
\definecolor{phaseDSS}{HTML}{956A3D}
\definecolor{phaseInk}{HTML}{333333}
\scalebox{0.7}{%
\begin{tikzpicture}[x=1cm,y=1cm,font=\normalsize,
  line cap=butt,line join=miter]
  \fill[black!7]
    (2.4,3.65) -- (2.4,4.15) -- (6.9,4.15) -- (8.9,2.72)
    -- (8.9,2.55) -- (5.1,2.55) -- (3.3,3.15) -- cycle;

  \draw[black!20,line width=0.3pt,densely dotted] (3.3,0) -- (3.3,3.15);
  \draw[black!20,line width=0.3pt,densely dotted] (5.1,0) -- (5.1,2.55);
  \draw[black!20,line width=0.3pt,densely dotted] (6.9,0) -- (6.9,4.15);
  \draw[black!20,line width=0.3pt,densely dotted] (8.9,0) -- (8.9,2.55);

  \draw[phasePurple,line width=1.5pt] (0,4.15) -- (2.0,4.15);
  \draw[phasePurple,line width=1.5pt] (2.4,4.15) -- (6.9,4.15);
  \draw[phaseDSS,line width=1.5pt]
    (6.9,4.15) -- (8.9,2.72) -- (11.8,0.82);
  \draw[phasePurple,line width=1.5pt,dash pattern=on 5pt off 3pt]
    (0,3.98) -- (2.0,3.98);
  \draw[phasePurple,line width=1.5pt,dash pattern=on 5pt off 3pt]
    (2.4,3.65) -- (3.3,3.15) -- (5.1,2.55);
  \draw[phaseBlue,line width=1.5pt,dash pattern=on 5pt off 3pt]
    (5.1,2.55) -- (8.9,2.55) -- (11.8,0.65);

  \node[phasePurple,anchor=south,align=center]
    at (4.35,4.25)
    {\cite{LWY24}\\[-1pt]
     Other upper bounds:\\[-1pt]
     \cite{GL21,LWY25,Lyu26}\\[-1pt]
     \cite{AFKT23,BF26}};
  \path (3.3,3.15) -- (5.1,2.55)
    node[midway,sloped,below=4pt,text=phasePurple] {\cite{LWY24}};
  \path (8.9,2.72) -- (11.8,0.82)
    node[midway,sloped,above=4pt,text=phaseDSS] {\cite{DSS24}};
  \node[phaseBlue,anchor=north] at (6.9,2.36)
    {Theorem~\ref{thm:main} (Ours)};

  \draw[phaseInk,line width=0.7pt] (0,0) -- (2.0,0);
  \draw[->,phaseInk,line width=0.7pt] (2.4,0) -- (12.3,0)
    node[right] {$\delta$};
  \foreach \phaseBreakY in {0,4.15} {
    \draw[phaseInk,line width=0.7pt]
      (2.08,\phaseBreakY-0.10) -- (2.19,\phaseBreakY+0.10);
    \draw[phaseInk,line width=0.7pt]
      (2.24,\phaseBreakY-0.10) -- (2.35,\phaseBreakY+0.10);
  }
  \draw[->,phaseInk,line width=0.7pt] (0,0) -- (0,5.35);
  \node[rotate=90,phaseInk] at (-2.05,2.9) {$\bE\bs{M_T}$};
  \draw[phaseInk] (-0.06,4.15) -- (0.06,4.15);
  \draw[phaseInk] (-0.06,2.55) -- (0.06,2.55);
  \node[anchor=east] at (-0.14,4.15) {$\log T$};
  \node[anchor=east] at (-0.14,2.55) {$\br{\log T}^{2/3}$};
  \foreach \phaseTick in {0,3.3,5.1,6.9,8.9,11.8}
    \draw[phaseInk] (\phaseTick,-0.06) -- (\phaseTick,0.06);
  \node[anchor=north] at (0,-0.17) {$0$};
  \node[anchor=north] at (1.15,-0.17) {$\delta\le T^{-a}$};
  \node[anchor=north] at (3.3,-0.17) {$e^{-\br{\log T}^{3/4}}$};
  \node[anchor=north] at (5.1,-0.17) {$e^{-\br{\log T}^{2/3}}$};
  \node[anchor=north] at (6.9,-0.17) {$\br{\log T}^{-1}$};
  \node[anchor=north] at (8.9,-0.17) {$\br{\log T}^{-2/3}$};
  \node[anchor=north] at (11.8,-0.17) {$1/4$};

\end{tikzpicture}
}
\endgroup
\caption{Expected total mistakes $\bE\bs{M_T}$ for $\Point$ at $\epsilon=1/100$.
Solid curves are upper bounds and dashed curves are lower bounds;
grey shading marks the open gap.}
\label{fig:delta-phase}
\end{figure}

\medskip
\noindent\textbf{General Littlestone Classes:} Our next result gives a
lower bound linear in $d$, for every class with $d\ge2$, while
retaining the polylogarithmic dependence on $T$ that we show for
$\Point$.

\begin{informaltheorem}
\label{thm:informal-dimension}
Let $\cH$ have finite Littlestone dimension $d\ge2$. For
$0<\epsilon\le1$, $0\le\delta\le\min\bc{\epsilon^2,1/4}$, and
$T\ge\br{d/\epsilon}^2$,
every $\br{\epsilon,\delta}$-differentially private online learner of
$\cH$ has a deterministic realisable stream of length $T$ on which
\[
  \bE\bs{M_T}=\Om{\min\bc{
    \frac d\epsilon\br{\log T}^{2/3},\;
    \frac d\delta
  }}.
\]
\end{informaltheorem}

\Needspace{6\baselineskip}
Intuitively, a shattered Littlestone tree supplies $d$ successive
stages in the non-private setting, at each of which a learner can be
forced to make a mistake. In our setting, the tree supplies
$\Theta(d)$ phases, each instantiating the hard construction
underlying Theorem~\ref{thm:main}. Following a single path through the
tree keeps the entire stream realisable, while the mistake
contributions from successive phases add. The construction is
visualised in~\Cref{fig:tree-to-stream}.

For fixed $d$ and $\epsilon$, this is, to our knowledge, the first
lower bound polynomial in $\log T$ that applies to every such class
and every private learner when $\delta=\Theta\br{1/\log
T}$.

\medskip
\noindent\textbf{Private online prediction} The lower bounds above
show that private online learning can require a number of mistakes
growing polynomially with $\log T$. We next consider the slightly
easier problem of private online prediction.  Here, the algorithm
observes $x_t$, releases only its prediction for $x_t$, and then
observes $y_t$. For point functions,~\cite{CLNSS24} already gives a
prediction mistake bound independent of the horizon $T$ and
polylogarithmic in $1/\delta$, but the corresponding question remains
open for general classes.

\begin{question}
Does every class of finite Littlestone dimension admit an
$(\epsilon,\delta)$-jointly private online predictor whose expected
mistake bound is independent of $T$ and polylogarithmic in $1/\delta$,
on every fixed realisable stream?
\end{question}

\noindent We answer this question affirmatively. Intuitively, in
online learning, the learner releases a hypothesis at every round. Its
predictions at informative points are therefore visible even on rounds
when it receives a dummy example. The learner receives informative
examples only occasionally, but its released hypotheses reveal
predictions at those points throughout the stream. This makes the
prediction setting harder for the adversary. In the prediction
setting, the privacy guarantee excludes the prediction at the round
where the neighbouring streams differ.  In online learning, the
hypothesis released at that round does not depend on the current
example, since it is released before that example is
observed.\looseness=-1

Below, we show that every class of finite Littlestone dimension admits
an expected mistake bound independent of $T$, with polylogarithmic
dependence on $1/\delta$.

\begin{informaltheorem}
\label{thm:informal-prediction}
Let $\cH$ be nonempty with finite Littlestone dimension $d$.
For $0<\epsilon\le1$ and $0<\delta<1/10$, there is an
$\br{\epsilon,\delta}$-jointly differentially private online predictor
such that, on every fixed realisable stream of length $T$,
\[
  \bE\bs{M_T}
  \le \frac{2^{2^{c\br{d+1}^2}}}{\epsilon^2}
       \log^2\frac{2}{\epsilon\delta},
\]
where $c>0$ is an absolute constant.
\end{informaltheorem}

For every fixed class of finite Littlestone dimension \(d\ge
2\),~\Cref{thm:informal-dimension,thm:informal-prediction} separate
private online learning from private online prediction.  At fixed
$\epsilon$ and $\delta=\Theta\br{1/\log T}$, the former requires
$\Om{\br{\log T}^{2/3}}$ expected mistakes on some stream, whereas the
latter admits $\bigO{\br{\log\log T}^2}$
expected mistakes on every realisable stream. The same separation also
holds for $\Point$ by Theorem~\ref{thm:informal-point}.

\subsection{Problem setting and preliminaries}
\label{sec:preliminaries}

\noindent\textbf{Setting} Let $\cH\subseteq\bc{0,1}^{\cX}$ be a
nonempty binary hypothesis class and let $T\ge1$ be the time horizon.
Write $S_r=\br{x_r,y_r}\in\cX\times\bc{0,1}$ for the labelled example
at round $r$, and $S=\br{S_1,\ldots,S_T}$ for the input stream.  The
stream $S$ is fixed independently of the algorithm's randomness.
This is the \emph{oblivious} setting in online learning. In this
work, we use replacement adjacency: two streams $S,S'$ of the same
length are \emph{neighbouring} if they differ in at most one labelled
example.
\begin{gamebox}
\begin{definition}[Online learning game]\label{def:online-game}
For each round $r=1,\ldots,T$:
\begin{enumerate}[label=\arabic*.,leftmargin=*,topsep=3pt,itemsep=0pt,parsep=0pt]
  \item The randomised learner $\cA$ releases a predictor
    $h_r:\cX\to\bc{0,1}$.
  \item The learner observes $(x_r,y_r)$.
\end{enumerate}
\end{definition}
\end{gamebox}
\noindent The predictors released by the learner need not belong to
$\cH$. We say the learner incurs a mistake when its prediction $\hat
y_r=h_r\br{x_r}$ is incorrect, i.e., when $\hat y_r\ne y_r$.

\begin{definition}[Differentially private online learning]
\label{def:private-online}
Write $\cA\br{S}=\br{h_1,\ldots,h_T}$ for the learner's complete
output. For $\epsilon\ge0$ and $0\le\delta<1$, the learner is
$\br{\epsilon,\delta}$-DP if, for every pair of neighbouring streams
$S,S'$ and every measurable output event $E$,
\[
  \bP\bs{\cA\br{S}\in E}
  \le e^\epsilon\bP\bs{\cA\br{S'}\in E}+\delta.
\]
\end{definition}
\noindent We next consider the slightly weaker online prediction model.

\begin{gamebox}
\begin{definition}[Online prediction game]\label{def:online-prediction}
For each round $r=1,\ldots,T$:
\begin{enumerate}[label=\arabic*.,leftmargin=*,topsep=3pt,itemsep=0pt,parsep=0pt]
  \item The predictor observes $x_r$ and then releases a bit $\hat
    y_r$.
  \item The predictor observes $y_r$.
\end{enumerate}
\end{definition}
\end{gamebox}

\begin{definition}[Jointly differentially private online prediction]
\label{def:joint-private-prediction}
Let $\hat Y\br{S}=\br{\hat y_1,\ldots,\hat y_T}$ and write
$\hat Y_{-i}\br{S}$ for this vector with coordinate $i$ removed.
For $\epsilon\ge0$ and $0\le\delta<1$, the predictor is
$\br{\epsilon,\delta}$-jointly differentially private if, for every
$i\in\bc{1,\ldots,T}$, every pair of streams $S,S'$ agreeing outside
round $i$, and every measurable output event $E$,
\[
  \bP\bs{\hat Y_{-i}\br{S}\in E}
  \le e^\epsilon\bP\bs{\hat Y_{-i}\br{S'}\in E}+\delta.
\]
\end{definition}
\noindent The stream is said to be \emph{realisable} by $\cH$ if there
is a single $f\in\cH$ such that $f\br{x_r}=y_r$ for every $r$.  In
either model, the total number of mistakes is
$M_T=\sum_{r=1}^T\mathbf{1}\bc{\hat y_r\ne y_r}$, where
$\mathbf{1}\bc{E}$ is the indicator of an event $E$.  On a fixed
stream, the expectation $\bE\bs{M_T}$ is computed over the algorithm's
randomness. We write $d=\Ldim\br{\cH}$~(see~\Cref{def:littlestone})
and call $\cH$ a Littlestone class when $d<\infty$. Our main lower
bound argument uses the following simple class.

\begin{definition}[Three-point class]\label{def:point-three}
On the instance space $\bc{1,2,3}$, let
\[
  \Point=\bc{f_1,f_2,f_3},
  \qquad f_j:\bc{1,2,3}\to\bc{0,1},
  \qquad f_j\br{x}=\mathbf{1}\bc{x=j}.
\]
\end{definition}

The class $\Point$ has Littlestone dimension one: both labels are
possible at each point, but prescribing label one leaves only one
hypothesis. The hypotheses $f_1,f_2$ agree at point $3$ and disagree
at point $1$. In our arguments, we choose a bit $B\in\bc{0,1}$
uniformly, independently of the learner's randomness. Then we denote
\(\br{3,0}\) as the  \emph{dummy example} and  \(\br{1,B}\) is the
\emph{informative example}. Every stream consisting of these two
examples is realisable: $f_1$ realises it when $B=1$, and $f_2$ when
$B=0$. Finally, all logarithms are natural unless another base is
specified. Terms involving $1/\delta$ are interpreted as $+\infty$
when $\delta=0$. Constants hidden by $O$, $\Omega$, and $\Theta$ are
absolute unless stated otherwise.

\subsection{Technical ideas}

\noindent\textbf{Lower bound for private online learning.} For
$\Point$, earlier work~\cite{DSS24} used the \emph{concentration}
assumption that the learner with high probability predicts the same
label at point \(1\) in an all-dummy stream. Then, the adversary
inserts informative examples with the opposite label, which causes the
learner to incur mistakes whenever it outputs its preferred label.
Privacy constrains how these insertions change the output
distribution. This allows the adversary to place informative examples
where mistakes are likely. Our goal is to force mistakes without this
concentration assumption.

We instead sample a uniformly random bit \(B\) at the beginning and
use it as the label for every informative example. On the dummy
stream, the prediction at point \(1\) is independent of \(B\), even if
the predictions keep changing (i.e. no concentration or similar
prior). Therefore, the first insertion of \(\br{1,B}\) causes a
mistake with probability \(1/2\) because the hypothesis is released
before the example arrives. However, this also reveals \(B\) to the
learner. The main challenge is therefore to control how much the
learner's predictions reveal about \(B\) before subsequent informative
examples.  Writing \(Y_r=h_r(1)\), we track
\(\Var\br{\bE\bs{B\mid Y_{1:r}}}\), the variance of the
posterior probability that \(B=1\). This measures the mean squared
deviation of the posterior probability of \(B=1\) from its prior value
of \(1/2\). If this variance can be kept small during rounds where
informative examples are shown, then the learner is forced to err on a
constant fraction of these examples in expectation.

We construct the entire stream of length \(T\) in \(L=\log_3 T\)
recursive stages agnostic of the learning algorithm. At each stage, we
choose a uniformly random third of the remaining interval. If we
choose the middle third, we insert an informative example at a
uniformly random position in the first third with probability \(p\).
We then recurse inside the chosen third. This gives \(Lp/3\)
informative examples in expectation.

We compare an insertion in the first third, which may affect
middle-third predictions, with a counterfactual insertion in the last
third, which cannot affect those predictions.
See~\Cref{fig:recursive-stage} for a visualisation of the construction
and the effect of the counterfactual insertion. Here we track a second
variance: how the effect of an informative point insertion on the
output distribution varies with the location of the insertion. Our
proof accounts for two competing effects. First, as mentioned above,
we continue the recursion inside a uniformly chosen third. By standard
variance decomposition, the average variance over insertion locations
within that third equals the original variance minus the variance
between thirds. We use this decrease to control the information
revealed about \(B\). Second, new insertions can add to the variance,
but privacy limits this increase.  Accounting for both effects over
the recursion controls how well the released predictions can identify
\(B\) before subsequent informative examples. The posterior variance
\(\Var\br{\bE\bs{B\mid Y_{1:r}}}\) before a possible
informative insertion is at most \(\bigO{p^2+Lp^3+Lp\delta}\).
Choosing \(p\) as a sufficiently small constant times
\(\min\bc{L^{-1/3},1/(L\delta)}\) keeps this variance small and yields
\(\Om{Lp}=\Om{\min\bc{(\log T)^{2/3},1/\delta}}\) expected mistakes.
We extend the lower bound to general Littlestone classes by roughly
applying the same construction at nodes along a path down a shattered
tree~(see~\Cref{fig:tree-to-stream}).

\noindent\textbf{Upper bound for private prediction.} For private
prediction, we first construct a short private sequence of predictors
that, with high probability, contains at least one with few mistakes
on the full stream. The first difficulty is to bound the number of bad
predictors in this list even when the hypothesis class is infinite.
Using the idea of~\emph{essential hypotheses}~\cite{Lyu26}, we create
multiple short lists of predictors with the useful property that
predictors can be added to these lists online using only earlier
mistakes. We then use a private selection mechanism to choose a
frequently appearing predictor among these lists. A contradiction
argument~(\Cref{lem:candidate-sequence}) then shows that the number of
bad predictors is bounded independently of \(T\).

The next challenge is to use these candidates privately online.
Each candidate can be computed from the observed prefix once its
predecessor has made enough mistakes.
However, switching immediately to a new candidate can be
privacy-violating. Unlike online learning, in online
prediction, only the prediction at the current point is released.
Thus, on dummy examples where the old and new predictors agree, the
learner can switch its internal predictor while continuing to give the
same answer that it would have given before the switch. 

We use a private replacement mechanism based on the Mirror
mechanism~\cite{CLNSS24} with increasing delays before each switch to
implement this in a private manner. The key property is that, while
the latest available candidate has made few mistakes, later candidates
cannot change the current predictions. Since these mistakes are
precisely what is needed to compute the next candidate, we can ignore
that candidate until enough mistakes are made or for the rest of the
game if this never happens.  The resulting online predictor has a
mistake bound \(Q\), independent of \(T\), that holds with probability
at least \(1-\delta\). However, this bound alone only guarantees
\(Q+\delta T\) expected mistakes. We therefore need one additional
step to get the bound in expectation. We combine our predictor with a
subsampled SOA algorithm~\cite{DSS24}, whose expected mistakes are
\(\bigO{d/\delta}\). Intuitively, this algorithm runs a competition
between the two predictors and guarantees that the combined predictor
competes with the better one in expectation, up to a constant factor
and an additive term independent of \(T\)
(\Cref{lem:prediction-combination}). This removes the remaining
dependence on \(T\).

\subsection{Related work}
\label{sec:related-work}

Private online learnability for every class of finite Littlestone
dimension was first shown in~\cite{GL21}. Against oblivious
adversaries, their realisable mistake bound has logarithmic dependence
on $T$ for fixed privacy parameters, but doubly-exponential dependence
on the Littlestone dimension $d$.  More
recently,~\cite{Lyu26} obtained the high-probability bound
$\softO{d^{9.5}\log T\log\br{1/\delta}/\epsilon}$, also against
oblivious adversaries. On $\Point$ at fixed $\epsilon$, they give
$\softO{\log T\log\br{1/\delta}}$. 

In contrast to oblivious adversaries, adaptive adversaries may choose
examples using the learner's previous outputs. In this
setting,~\cite{LWY25} gives a mistake bound
$\bigO{2^{2^{\bigO{d}}}\br{\log T+\log\br{1/\delta}}/\epsilon}$. For
$\Point$ at fixed $\epsilon$, this is $\bigO{\log\br{T/\delta}}$, and
hence $\bigO{\log T}$ whenever $\delta\ge T^{-a}$ for fixed $a>0$. The
same work also establishes sublinear regret for general Littlestone
classes in the agnostic setting, where the labels need not be
realisable. Our work does not consider the agnostic or the adaptive
setting but rather focuses on the oblivious setting with realisable
streams.

Recent works on lower bounds for private counting have also used
similar techniques to obtain polylogarithmic lower bounds.
\cite{HL26} used the same variance decomposition and recursion, while
\cite{BDL26} used a related argument to accumulate noise along a tree.
The main difference in this regard between our work and theirs is the
probability event to which these arguments are applied. Importantly,
continual counting with \(\ell_\infty\) error is known to be related
to our problem of online learning via a reduction from private
counters to private online learners for point
functions~\cite{DSS24,CLNSS24}. Thus, online-learning lower bounds
imply counting lower bounds, with weaker parameters. The counting
lower bounds do not, however, directly imply our online-learning
lower bound.

Online prediction under joint differential privacy releases only the
answer to the current query, and does not require this answer to hide
the current example.  For general Littlestone classes,~\cite{KMMNS23}
gives the mistake bound
$\softO{d^2\epsilon^{-2}\log^2\br{1/\delta}\log^2\br{T/\beta}}$ with
probability at least $1-\beta$. Their result allows adaptive inputs
and satisfies the stronger challenge-privacy definition of that work.
Our guarantee concerns oblivious streams and joint privacy, and
removes the dependence on $T$. For point functions on any
domain,~\cite{CLNSS24} gives $\bigO{\epsilon^{-2}\log^2\br{1/\delta}}$
prediction mistakes, independently of the horizon $T$. However, their
result does not apply to every hypothesis class with finite
Littlestone dimension. To the best of our knowledge, our result is the
first horizon-independent guarantee with polylogarithmic dependence
on $1/\delta$ for every class of finite Littlestone dimension.

\section{Lower Bounds for Online Learning}
\label{sec:lower-bounds}
\subsection{Lower bound for POINT}
\label{sec:main-results}

We first prove the lower bound for $\Point$. Our strategy is to place
informative examples at randomly chosen positions in the stream and
show that the learner errs on a constant fraction of them in
expectation.\looseness=-1

\begin{restatable}{theorem}{PointLowerBound}\label{thm:main} For \(\epsilon=1/100,
0\le\delta\le1/4\), and every $\br{\epsilon,\delta}$-differentially
private online learner of $\Point$, there is a realisable stream of
length $T$ on which
\[
  \bE\bs{M_T}
  \ge c\min\bc{\br{\log\br{2T}}^{2/3},\frac1\delta},
\]
where $c>0$ is an absolute constant.
\end{restatable}

\begin{proof}[Proof of~\Cref{thm:main}]
  For $T=1,2$, use $(1,B)$ at the first round, where $B$ is a fair
  bit, and leave any remaining round dummy. The expected number of
  mistakes at the first round is $1/2$, so some fixed choice of $B$
  gives the claimed bound after decreasing $c$ if necessary. We
  henceforth assume $T\ge3$.

  We first describe the adversary's strategy for a parameter
  \(p\in(0,1]\). The adversary initialises a stream of length \(T\)
  consisting entirely of dummy examples \(\br{3,0}\) and samples a bit
  \(B\in\bc{0,1}\) uniformly at random. Then, starting from the dummy
  stream, the adversary proceeds in $L=\floor{\log_3 T}$ stages,
  starting with the interval $A_0=(0,3^L]$ and leaving all later
  rounds dummy. In each stage, the adversary partitions the remaining
  interval into three equal parts and chooses one of the thirds
  uniformly at random.
  If the middle third is chosen, then with probability \(p\), the
  adversary replaces a uniformly chosen dummy example in the first
  third with an informative example \(\br{1,B}\); otherwise, it does
  nothing. The adversary then proceeds to the next stage by
  restricting to the chosen interval and repeating the process.
  See~\Cref{fig:recursive-stage}(a) for an illustration and
  \Cref{alg:adversary} for a formal description.

  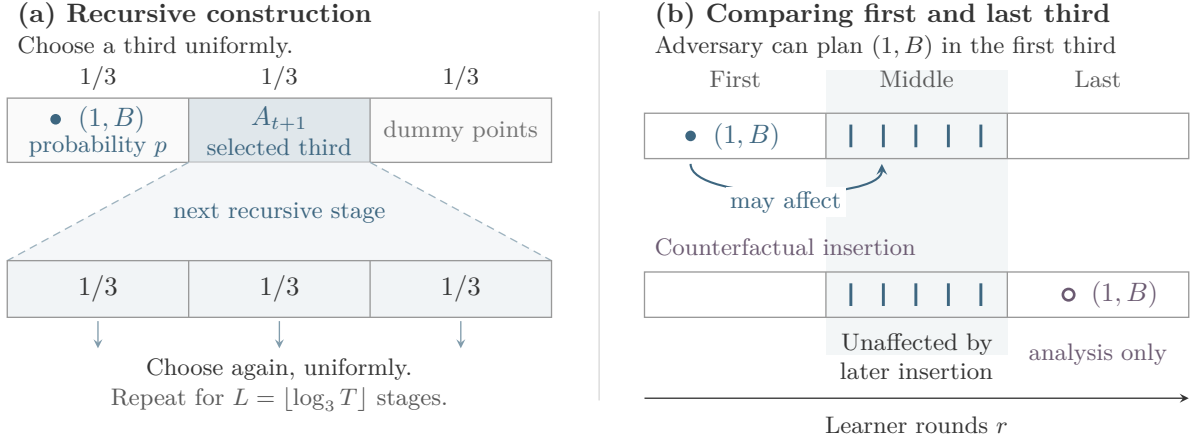
\begin{figure}[!t]
\centering
\definecolor{stageBlue}{HTML}{3D6680}
\definecolor{stagePurple}{HTML}{6B5A71}
\definecolor{stageInk}{HTML}{333333}
\begin{tikzpicture}[x=1cm,y=1cm,font=\small,text=stageInk,
  line width=0.5pt,>=stealth,
  heading/.style={anchor=west,font=\small\bfseries},
  note/.style={font=\footnotesize,align=center},
  edge/.style={draw=black!40},
  prediction/.style={stageBlue,line width=1.1pt}]
  \node[heading] at (0,5.1) {(a) Recursive construction};
  \node[anchor=west,note] at (0,4.68)
    {Choose a third uniformly.};
  \foreach \x in {1.2,3.6,6}
    \node[note] at (\x,4.26) {$1/3$};
  \fill[black!2] (0,3.15) rectangle (7.2,4);
  \fill[stageBlue!16] (2.4,3.15) rectangle (4.8,4);
  \draw[edge] (0,3.15) rectangle (7.2,4);
  \foreach \x in {2.4,4.8}
    \draw[edge] (\x,3.15) -- (\x,4);
  \fill[stageBlue] (0.64,3.72) circle (0.07);
  \node[anchor=west,text=stageBlue] at (0.8,3.72) {$(1,B)$};
  \node[note,text=stageBlue] at (1.2,3.36) {probability $p$};
  \node[text=stageBlue] at (3.6,3.72) {$A_{t+1}$};
  \node[note,text=stageBlue] at (3.6,3.36) {selected third};
  \node[note,text=black!55] at (6,3.57) {dummy points};

  \fill[stageBlue!5]
    (2.4,3.15) -- (4.8,3.15) -- (7.2,1.85) -- (0,1.85) -- cycle;
  \draw[stageBlue!50,densely dashed] (2.4,3.15) -- (0,1.85);
  \draw[stageBlue!50,densely dashed] (4.8,3.15) -- (7.2,1.85);
  \node[note,text=stageBlue] at (3.6,2.48) {next recursive stage};
  \fill[stageBlue!7] (0,1.1) rectangle (7.2,1.85);
  \draw[edge] (0,1.1) rectangle (7.2,1.85);
  \foreach \x in {2.4,4.8}
    \draw[edge] (\x,1.1) -- (\x,1.85);
  \foreach \x in {1.2,3.6,6} {
    \node at (\x,1.48) {$1/3$};
    \draw[->,stageBlue!65] (\x,1.05) -- (\x,0.69);
  }
  \node[note] at (3.6,0.42) {Choose again, uniformly.};
  \node[note,text=black!65] at (3.6,0.02)
    {Repeat for $L=\floor{\log_3 T}$ stages.};

  \draw[black!15] (7.83,0) -- (7.83,5.15);

  \node[heading] at (8.43,5.1) {(b) Comparing first and last third};
  \node[anchor=west,note] at (8.43,4.68)
    {Adversary can plan \(\br{1,B}\) in the first third};
  \fill[stageBlue!6] (10.83,0.6) rectangle (13.23,4.37);
  \foreach \x/\label in {9.63/First,12.03/Middle,14.43/Last}
    \node[note,text=black!65] at (\x,4.26) {\label};
  \draw[edge] (8.43,3.2) rectangle (15.63,3.8);
  \foreach \x in {10.83,13.23}
    \draw[edge] (\x,3.2) -- (\x,3.8);
  \fill[stageBlue] (9.04,3.5) circle (0.07);
  \node[anchor=west,text=stageBlue] at (9.2,3.5) {$(1,B)$};
  \foreach \x in {11.16,11.59,12.02,12.45,12.88}
    \draw[prediction] (\x,3.35) -- (\x,3.65);
  \draw[->,stageBlue,line width=0.8pt]
    (9.04,3.13) .. controls (9.04,2.63) and (11.57,2.63) .. (11.57,3.13);
  \node[note,text=stageBlue,fill=white,inner sep=1.5pt]
    at (10.3,2.62) {may affect};

  \node[anchor=west,note,text=stagePurple] at (8.43,2.04)
    {Counterfactual insertion};
  \draw[edge] (8.43,1.1) rectangle (15.63,1.7);
  \foreach \x in {10.83,13.23}
    \draw[edge] (\x,1.1) -- (\x,1.7);
  \foreach \x in {11.16,11.59,12.02,12.45,12.88}
    \draw[prediction] (\x,1.25) -- (\x,1.55);
  \draw[stagePurple,line width=1pt,fill=white]
    (14.04,1.4) circle (0.075);
  \node[anchor=west,text=stagePurple] at (14.2,1.4) {$(1,B)$};
  \node[note,text width=2.4cm] at (12.03,0.59)
    {Unaffected by\\later insertion};
  \node[note,text=stagePurple] at (14.43,0.62) {analysis only};
  \draw[->,stageInk] (8.43,0.03) -- (15.63,0.03);
  \node[note,anchor=north] at (12.03,-0.08) {Learner rounds $r$};
\end{tikzpicture}
\caption{In (a), we recurse inside a uniformly chosen third.  If we
choose the middle third, we insert an informative example at a
uniformly random position in the first third with probability $p$;
otherwise, we make no insertion.  In (b), we compare this insertion
with a counterfactual insertion in the last third. The first may
affect the middle-third predictions (vertical marks), whereas the
counterfactual insertion cannot, since those predictions are made
before the example arrives.}
\label{fig:recursive-stage}
\end{figure}

  The expected number of informative examples in the stream is
  \(Lp/3\). The online learner releases a hypothesis
  \(h_r\) at each round \(r\in\bc{1,\ldots,T}\); let
  \(Y_r=h_r\br{1}\).

  For each \(t\), consider running the online learner on the stream
  constructed by the adversary after the first \(t\) stages above with
  no subsequent insertions. Conditional on the adversary's choices of
  thirds through stage \(t\),\footnote{Here \(t\in\bc{0,\ldots,L}\)
  counts completed recursive stages, and \(r\in\bc{1,\ldots,T}\)
  indexes the learner's rounds.} define
  \(V_t\br{r}=\bE\bs{\br{\bE\bs{2B-1\mid Y_1,\ldots,Y_r}}^2}\), with
  \(V_t\br{0}=0\), to measure how much the learner's outputs through
  round \(r\) reveal about \(B\).  Let \(A_t=(a_t,b_t]\) be the
  interval remaining after \(t\) stages, and define
  \[
    I_t=V_t\br{a_t}+2\br{V_t\br{b_t}-V_t\br{a_t}}.
  \]
  Here \(\bE\bs{I_t}\) averages over the adversary's choices of
  thirds. The following lemma shows that, if \(\bE\bs{I_t}\) is
  small, the learner makes mistakes on a constant fraction of the
  informative examples inserted at stage \(t+1\), in expectation.
  \begin{restatable}{lemma}{MistakeContribution}\label{lem:mistakes}
The expected number of mistakes contributed by stage $t+1$ is at least
\(
  \frac p6\br{1-\sqrt{\bE\bs{I_t}}}.
\)
\end{restatable}
An immediate consequence of~\Cref{lem:mistakes} is that, if
$\bE\bs{I_t}\le1/4$ for every $0\le t<L$, the expected number of
mistakes on informative examples is at least $Lp/12$. We next show that
\begin{equation}\label{eq:target}
\bE\bs{I_t}\le C\br{p^2 + tp^3 + tp\delta}
\end{equation}
for a universal constant \(C\). Then, choosing
\(p=c_0\min\bc{L^{-1/3},1/\br{\delta L}}\) for a sufficiently small
\(c_0>0\) ensures that \(\bE\bs{I_t}\le1/4\) for every \(0\le t<L\)
and
\[
  \bE\bs{M_T}\ge \frac{Lp}{12}
  =\Om{\min\bc{L^{2/3},1/\br{\delta}}}.
\]
It remains to prove~\eqref{eq:target}. 

So far, we have not used the fact that the online learner is
differentially private. Fix the adversary's choice of the thirds
through stage \(t\) and \(b\in\bc{0,1}\). Let \(P_\varnothing^b\)
denote the distribution of \(Y=\br{Y_1,\ldots,Y_T}\), conditional on
\(B=b\), when the learner runs on the stream constructed through stage
\(t\), with no further insertions. For \(s\in A_t\) where
\(A_t\coloneqq (a_t,b_t]\), let \(P_s^b\) denote the corresponding
distribution after additionally replacing the dummy example at round
\(s\) with \(\br{1,b}\), leaving all other entries unchanged. For each
fixed realisation of the earlier insertions, the streams defining
\(P_\varnothing^b\) and \(P_s^b\) differ only at round \(s\). For
every \(y\) with \(P_\varnothing^b(y)>0\), define \(\ell_s^b(y)\) to
be \(P_s^b(y)/P_\varnothing^b(y)\) clipped to
\([e^{-\epsilon},e^\epsilon]\). Differential privacy gives
\[
  \norm{P_s^b-P_\varnothing^b\ell_s^b}_1\le2\delta.
\]

We now define two quantities \(W_t\) and \(D_t\) to measure the
variance of this ratio within one stage \(A_t\) and the difference
between the average ratios of the first and last thirds of \(A_t\),
respectively:
\[
 W_t=\bE\bs{\Var_{s\in A_t}(\ell_s^B)},
\qquad
D_t=\bE\bs{\br{
  \overline\ell_{A_t^1}^B-\overline\ell_{A_t^3}^B}^2}.
\]
Here the expectation is over the joint distribution of \(\br{B,Y}\)
and \(\overline\ell_{A_t^1}^B,\overline\ell_{A_t^3}^B\) are the
averages of \(\ell_s^B\) over \(s\in A_t^1\) and \(s\in A_t^3\),
respectively. Compare an insertion in the first third with a
hypothetical insertion in the last third, as illustrated in
Figure~\ref{fig:recursive-stage}(b). The latter cannot affect
predictions in the middle third, since the learner has not yet
received that example. The quantity \(D_t\) measures the squared
difference between the effects of these two insertions.

Suppose the middle third is chosen, so \(b_{t+1}\) is its right
endpoint. For any two round indices \(a_t\le u\le v\le b_{t+1}\),
comparing the reduction in the best expected squared error for
estimating $2B-1$ before and after the insertion gives
\[
  V_{t+1}(v)-V_{t+1}(u)
  \le 2\br{V_t(v)-V_t(u)}+8p^2D_t+16p\delta.
\]

Inserting
with probability \(p\) multiplies the first-versus-last difference
by \(p\), giving \(p^2D_t\) after squaring. The \(p\delta\) term
accounts for clipping. 
Averaging over
the three choices and using the definition of \(I_t\) gives the
first inequality in \Cref{lem:variance-recursion}.

On the other hand, a large \(D_t\) means that a substantial part of
the variance in \(W_t\) comes from the differences between the thirds.
Restricting to within a third removes this between-third variation,
and therefore reduces the variance within a third. Variance
decomposition gives
\[
  \frac13\sum_{j=1}^3
  \bE\bs{
    \Var_{s\in A_t^j}(\ell_s^B)}
  \le W_t-\frac{D_t}{6}.
\]
We must also account for the possible insertion when the middle
third is chosen. Using privacy again and convexity of variance, the
increase is at most \(p\) times the variance of the ratios on the
insertion branch, plus the clipping error \(\delta\). These ratios
lie in \([e^{-\epsilon},e^\epsilon]\), so their variance is at most
\((e^\epsilon-e^{-\epsilon})^2/4\). Averaging over the three choices
gives
\[
\begin{aligned}
  &\bE\bs{W_{t+1}\mid J_{\le t}}
  -\frac13\sum_{j=1}^3
  \bE\bs{
    \Var_{s\in A_t^j}(\ell_s^B)
  }\le
  \frac p3\cdot\frac{(e^\epsilon-e^{-\epsilon})^2}{4}
  +\frac{\delta}{3}\le\frac{p+\delta}{3},
\end{aligned}
\]
where the last inequality uses \((e^\epsilon-e^{-\epsilon})^2\le4\).
Combining these bounds gives the second inequality in the following
lemma. See the proof of \Cref{lem:variance-recursion} for the full details.

\begin{restatable}{lemma}{VarianceRecursion}\label{lem:variance-recursion} For every \(0\le t< L\)
\begin{equation}\label{eq:variance-bound}
  \bE\bs{I_{t+1}\mid J_{\le t}}
  \le I_t+20p^2D_t+20p\delta\end{equation}
  \begin{equation}\label{eq:variance-bound-2}
  \bE\bs{W_{t+1}\mid J_{\le t}}
  \le W_t-\frac{D_t}{6}+\frac{p+\delta}{3}.
\end{equation}
where \(J_{\le t}\) denotes the adversary's choices of thirds through stage \(t\).
\end{restatable}
Now multiplying~\eqref{eq:variance-bound-2} above by \(120p^2\),
adding it to~\eqref{eq:variance-bound}, and taking an expectation on
both sides gives 
\[
 \bE\bs{I_{t+1}}\le \bE\bs{I_t}+ 120p^2\br{\bE\bs{W_{t}}
- \bE\bs{W_{t+1}}} + 40p^3 + 60p\delta\]
Summing this inequality for \(t=0,\ldots,m-1\), where \(1\le m\le L\),
and using \(I_0=0\), \(W_0\le1\), and \(W_m\ge0\), gives
\[\bE\bs{I_m}\le 120p^2 + 40mp^3 + 60mp\delta\leq 120\br{p^2 + mp^3 +mp\delta}.\] This completes the
proof of \eqref{eq:target} with \(C=120\) and hence
of~\Cref{thm:main}.
\end{proof}

\subsection{Lower bound for general Littlestone classes}
\label{sec:dimension-extension}

We now extend the lower bound to general Littlestone classes. The main
idea is to apply the two-label construction at successive nodes of a
shattered tree, following~\cite{DSS24,LWY24}.

\begin{restatable}{theorem}{DimensionLowerBound}\label{thm:dimension}
Let $\cH\subseteq\bc{0,1}^{\cX}$ have finite Littlestone dimension
$d\ge2$, and let $T\ge1$. Suppose that $0<\epsilon\le1$ and
$0\le\delta\le\min\bc{\epsilon^2,1/4}$. For every
$\br{\epsilon,\delta}$-differentially private online learner of $\cH$,
there is a deterministic realisable stream of length $T$ on which
\[
  \bE\bs{M_T}
  \ge c\min\bc{
    T,\;
    \frac d\epsilon\bs{\log\br{2+\frac{\epsilon T}{d}}}^{2/3},\;
    \frac d\delta
  },
\]
where $c>0$ is an absolute constant.
\end{restatable}

\begin{proof}[Proof sketch]
We give the proof in the regime
\(
  T\ge\br{d/\epsilon}^2,\delta\le\frac{\epsilon}{\br{\log T}^{2/3}}.
\) The full proof, covering all parameter ranges, is given in~\Cref{app:dimension-extension}.

Take a depth-$d$ tree shattered by $\cH$ and follow the zero (left)
edge at its root $x_0$. The remaining subtree has depth $d-1$, and
every hypothesis realising one of its paths agrees on $\br{x_0,0}$.
We use this as the dummy example throughout. Divide the horizon into
$k=d-1$ phases of length $n=\floor{T/k}$, leaving any unused rounds
dummy. See~\Cref{fig:tree-to-stream} for an illustration.

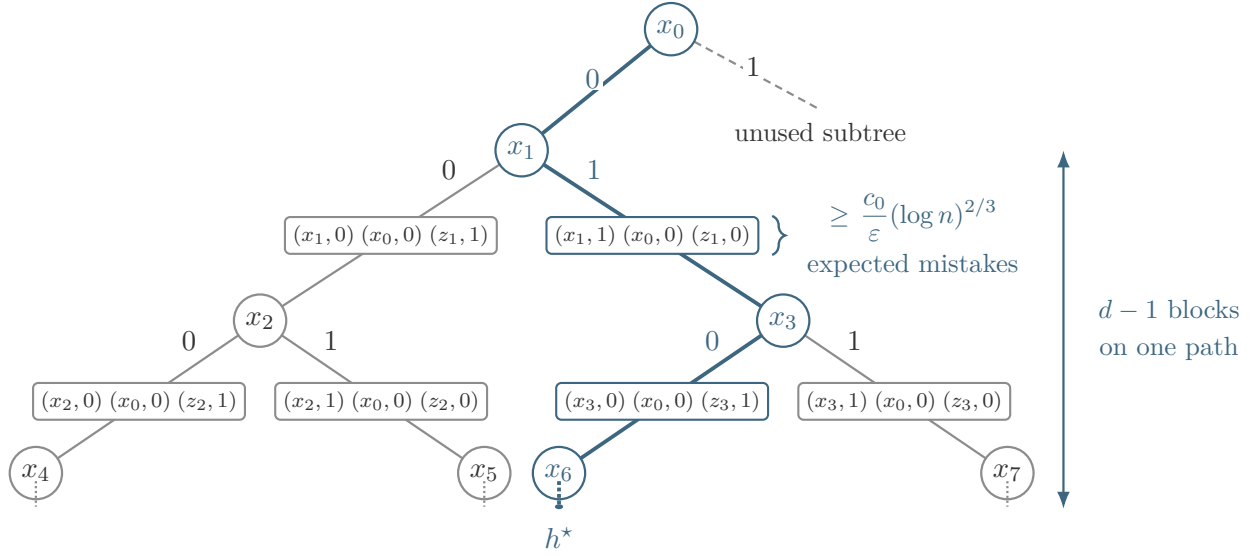
\begin{figure}[!t]
\centering
\definecolor{ink}{HTML}{333333}
\colorlet{quiet}{black!45}
\colorlet{muted}{black!55}
\definecolor{accent}{HTML}{3D6680}
\newcommand{\pointblock}[6]{%
  \node[draw=#3,fill=white,rounded corners=2pt,line width=.8pt,
        minimum width=2.75cm,align=center,
        inner sep=3pt,text=ink,font=\scriptsize] (#4) at (#1,#2)
    {$(x_{#5},#6)\;(x_0,0)\;(z_{#5},\number\numexpr1-#6\relax)$};
}
\resizebox{\linewidth}{!}{%
\begin{tikzpicture}[x=1cm,y=.60cm,
  vertex/.style={circle,draw=quiet,fill=white,minimum size=7mm,
                 line width=.85pt,inner sep=2pt},
  edge/.style={draw=quiet,line width=.85pt},
  chosen/.style={draw=accent,line width=1.45pt},
  bit/.style={font=\normalsize,inner sep=1pt,fill=white},
  every node/.style={text=ink}]
\path[use as bounding box] (-6.9,-8.75) rectangle (9.8,3.15);

\coordinate (original-root) at (2,2.7);
\coordinate (discarded) at (4,.9);
\coordinate (top) at (0,0);
\coordinate (left) at (-3.5,-3.8);
\coordinate (right) at (3.5,-3.8);
\coordinate (ll) at (-6.5,-7.2);
\coordinate (lr) at (-.5,-7.2);
\coordinate (rl) at (.5,-7.2);
\coordinate (rr) at (6.5,-7.2);

\draw[chosen] (original-root) -- (top);
\draw[edge,densely dashed] (original-root) -- (discarded);
\node[bit,text=accent] at (.95,1.6) {$0$};
\node[bit] at (3.1,1.85) {$1$};
\node[font=\small,below=1pt] at (discarded) {unused subtree};
\node[vertex,draw=accent,text=accent] at (original-root) {$x_0$};

\draw[edge] (top) -- (left);
\draw[chosen] (top) -- (right);
\draw[edge] (left) -- (ll);
\draw[edge] (left) -- (lr);
\draw[chosen] (right) -- (rl);
\draw[edge] (right) -- (rr);

\node[bit] at (-.97,-.45) {$0$};
\node[bit,text=accent] at (.97,-.45) {$1$};
\node[bit] at (-4.45,-4.25) {$0$};
\node[bit] at (-2.55,-4.25) {$1$};
\node[bit,text=accent] at (2.55,-4.25) {$0$};
\node[bit] at (4.45,-4.25) {$1$};

\pointblock{-1.75}{-1.9}{quiet}{block-left}{1}{0}
\pointblock{1.75}{-1.9}{accent}{block-right}{1}{1}
\pointblock{-5.12}{-5.6}{quiet}{block-ll}{2}{0}
\pointblock{-1.88}{-5.6}{quiet}{block-lr}{2}{1}
\pointblock{1.88}{-5.6}{accent}{block-rl}{3}{0}
\pointblock{5.12}{-5.6}{quiet}{block-rr}{3}{1}

\draw[draw=accent,line width=.9pt,
      decorate,decoration={brace,amplitude=5pt}]
  ($(block-right.north east)+(.17,0)$) --
  ($(block-right.south east)+(.17,0)$)
  node[midway,right=9pt,text width=2.85cm,align=center,
       text=accent,font=\small]
    {$\displaystyle\ge\frac{c_0}{\epsilon}(\log n)^{2/3}$\\[4pt]
     expected mistakes};

\node[vertex,draw=accent,text=accent] at (top) {$x_1$};
\node[vertex] at (left) {$x_2$};
\node[vertex,draw=accent,text=accent] at (right) {$x_3$};
\node[vertex] at (ll) {$x_4$};
\node[vertex] at (lr) {$x_5$};
\node[vertex,draw=accent,text=accent] at (rl) {$x_6$};
\node[vertex] at (rr) {$x_7$};
\foreach \pos in {ll,lr,rr} {
  \draw[edge,densely dotted] ($(\pos)+(0,-.18)$) -- ($(\pos)+(0,-.75)$);
}
\draw[chosen,densely dotted] ($(rl)+(0,-.18)$) -- ($(rl)+(0,-.75)$);
\fill[accent] ($(rl)+(0,-.75)$) circle[radius=.055];
\node[below=4pt,text=accent] at ($(rl)+(0,-.75)$) {$h^\star$};

\draw[draw=accent,line width=.85pt,
      {Latex[length=2.3mm]}-{Latex[length=2.3mm]}]
  (7.25,0) -- (7.25,-8.0)
  node[midway,right=6pt,text width=2.12cm,align=center,
       text=accent,font=\small]
    {$d-1$ blocks\\[2pt]on one path};

\end{tikzpicture}%
} \caption{A shattered Littlestone tree with $d-1$ hard \emph{blocks}
along a root-to-leaf path after the initial zero edge. Each block
applies the $\mathrm{POINT}_3$ construction from~\Cref{thm:main} using
the parent point $x_i$ as its informative point and $x_0$ as its
shared dummy. The point $z_i$ is the unused third point of this
construction and never appears in the stream~(as in~\Cref{thm:main}).
The informative label \(B\) is the label of the edge.} 
\label{fig:tree-to-stream}
\end{figure}

We use the following consequence of Section~\ref{sec:main-results};
see Lemma~\ref{lem:one-bit-general-privacy} for the extension to
general $\epsilon$. When $n\ge\max\bc{2,\epsilon^{-2}}$ and
$\delta\le\epsilon/\br{\log n}^{2/3}$, for any informative point
$x\ne x_0$, every $\br{\epsilon,\delta}$-private online learner has
a deterministic stream of length $n$ consisting only of the dummy
example and $\br{x,b}$ for one fixed bit $b$, on which its expected
mistakes on informative examples are at least
\[
  \frac{c_0}{\epsilon}\br{\log n}^{2/3}.
\]
Our choices of $n$ and the parameter regime above satisfy these
conditions. At the current node $x$, both labels are available. Fix
the input from earlier phases and temporarily fill all later rounds
with dummy examples. The original learner induces an
$\br{\epsilon,\delta}$-private learner on the current phase: replacing
one example in the phase replaces only one example in the full stream.
We fix only the input prefix and average over the learner's
randomness. Apply the preceding bound at $x$ to choose a hard phase
and its label $b$, then follow the edge labelled $b$.  Repeat for all
$k$ phases.  Figure~\ref{fig:tree-to-stream} illustrates these phases
along a path in the shattered tree.

The phases are chosen using the learner's distribution, so the entire
stream is fixed before the learner runs. Their labels follow one
path in the shattered tree, making the stream realisable. Moreover,
replacing the temporary dummy suffix with later phases cannot change
earlier predictions. Adding the expected mistakes over the phases
gives
\[
  \bE\bs{M_T}
  \ge\frac{c_0k}{\epsilon}\br{\log n}^{2/3}
  \ge\frac{c_1d}{\epsilon}\br{\log T}^{2/3},
\]
where the last inequality uses $k\ge d/2$ and $n\ge\sqrt T$.
This implies the stated bound in the parameter regime considered.
\end{proof}

\section{Private predictor for Littlestone classes}
\label{sec:private-candidates}

Let $\cH\subseteq\bc{0,1}^{\cX}$ be nonempty with Littlestone dimension
$d<\infty$. Fix $0<\epsilon\le1$ and $0<\delta<1/10$.  In this
section, we consider the online prediction setting with joint DP, as
described in~\Cref{sec:preliminaries}. Let
$S=\br{\br{x_t,y_t}}_{t=1}^T$ be a fixed labelled stream and
$M\br{U;S_{1:r}}$ be the mistakes of a prediction sequence $U$ on its
first $r$ rounds. Our main result of this section is the following.

\begin{restatable}{theorem}{PredictionBound}\label{thm:online-prediction-expectation} There is an $(\epsilon,\delta)$-JDP online predictor satisfying
\[
  \bE\bs{M\br{\cA;S}}
  \le \frac{2^{2^{c(d+1)^2}}}{\epsilon^2}
       \log^2\frac{2}{\epsilon\delta}
\]
on every fixed realisable stream of any length, where $c>0$ is an
absolute constant.
\end{restatable}

\begin{proof}
If $d=0$, a fixed hypothesis in $\cH$ predicts without mistakes.
Assume $d\ge1$, and consider a realisable stream $S$.

Intuitively, in online learning, the learner releases a hypothesis at
every round. Its predictions at informative points are therefore
visible even on rounds when it receives a dummy example. In online
prediction, only the prediction at the current point is released.
Thus, on dummy examples where the old and new predictors agree, the
learner can change its internal predictor while continuing to give the
same answer. Changes in its predictions at informative points are not
visible through these outputs. We use this freedom to delay the effect
of a new candidate on the released predictions. Additionally, to bound the mistakes
from these delays, we also need to bound how many candidates are used.

The proof proceeds in three steps. First, we construct a short
private sequence of candidates containing one with few mistakes on
the whole stream, with high probability. Both the length of the
sequence and the mistake bound are independent of $T$. Each new
candidate can be computed from a bounded number of mistakes of
the preceding candidates.

Second, we use these candidates for private online prediction.
Simply switching when a new candidate is computed may reveal when
its predecessor made mistakes. Our replacement rule returns the
common prediction when the old and new predictors agree, and
privately delays the use of the new candidate on disagreement
rounds. We choose the delays so that a candidate cannot affect
the predictions before it can be computed. The resulting mistake
bound depends on the number of candidates and the delays, both
independent of $T$.

In the third step, we convert this high-probability guarantee into
an expected mistake bound.

\noindent\textbf{Step 1: construct a private sequence containing a good
candidate.} We adapt the essential-hypothesis construction
of~\cite{Lyu26} to obtain the following guarantee.

\begin{restatable}{lemma}{CandidateSequence}\label{lem:candidate-sequence}
The offline algorithm in \Cref{alg:private-candidates}, run with privacy
parameters $(\epsilon/4,\delta/8)$ and failure probability $\beta=\delta$,
is $(\epsilon/4,\delta/8)$-differentially private and returns binary
functions $h_0,\ldots,h_N$. Then, on every fixed realisable stream,
\[
  \bP\bs{\min_{0\le j\le N}M\br{h_j;S}<K}\ge1-\delta,
\]
where $h_0$ is public and $N,K$ are positive integers satisfying
\(
  N\le2^{\bigO{d^2}}, K\le\frac{2^{\bigO{d^2}}}{\epsilon}\log\frac{2}{\epsilon\delta}.
\)
\end{restatable}

We explain the construction and its mistake guarantee here, using the
lemmas proved in \Cref{app:private-candidates}. For a function $h$,
let $B_K(h,S)$ contain the first $K$ labelled examples on which $h$
makes a mistake, or all such examples if there are fewer. Call the
batch full when it contains $K$ records.

To compute $h_{j+1}$, the algorithm only needs
$B_K(h_0,S),\ldots,B_K(h_j,S)$. Thus, once
$B_K(h_j,S)$ becomes full, we can compute the next candidate using
the observed examples and the state from the preceding batches.

 For a nonempty class $G\subseteq\cH$ and an integer
$p\ge1$, let $C_p(G)$ be its set of \emph{essential hypotheses}
(\Cref{def:essential-hypotheses}) and let $\rho_p(G)\in\bc{0,\ldots,d}$ be
its \emph{decomposition dimension} (\Cref{def:decomposition-dimension}).
These notions are due to~\cite{Lyu26}.

The important property is that each of these essential hypotheses agrees
with some function in $G$ on any given set of at most $p$ instances and
$|C_p(G)|\le p^d2^{d^2}$. The main property of decomposition dimension is
that, if $\varnothing\ne F\subseteq G$ and $p'\in\bc{p,2p}$, then
$\rho_{p'}(F)\le\rho_p(G)$, and if equality holds, then $C_{p'}(F)\subseteq
C_p(G)$. Thus restricting the class cannot increase its decomposition
dimension, even when $p$ is doubled, and if the dimension is unchanged, its
essential hypotheses can only shrink.

With the parameters specified before Algorithm~\ref{alg:private-candidates},
the algorithm maintains $k$ groups of labelled examples from $S$.
Figure~\ref{fig:private-candidates} illustrates how a mistake batch
updates their stored samples and candidate lists. Starting with a public function $h_0$, it constructs the next candidate
$h_j$ by forming the mistake batch $B_K(h_{j-1},S)$ and assigning each
example independently to a uniformly chosen group. Each group stores a
uniformly chosen sample of up to $n$ of the examples it receives, along with
all its earlier samples. The stored samples define the class
$G_i\subseteq\cH$ for group $i$: a hypothesis belongs to $G_i$ if it makes
at most $a_\ell n$ mistakes on every stored sample, where
\(
  a_\ell=\frac18-\frac{\ell}{32(d+1)}, 0\le\ell\le d.
\)

The realising hypothesis remains in every $G_i$, so these classes are
nonempty. Initially $\ell=0$, and group $i$'s candidate list is $C_p(G_i)$.
We select a hypothesis using the number of lists in which it
appears. The highlighted copies of $h_j$ in
Figure~\ref{fig:private-candidates} represent the same function
appearing in different groups' lists. Changing one group's data
can change its entire list, but changes each such count by at
most one. A private sparse
selection algorithm (\Cref{fact:candidate-selection}) uses these
counts to choose a hypothesis appearing in many lists or return
$\bot$.

 If it returns a hypothesis,
this becomes $h_j$. If it returns $\bot$ and $\ell<d$, the algorithm
increments $\ell$ and retries using the same stored samples. Then,~we
show in~\Cref{lem:candidate-sampling,lem:candidate-selection} that
with probability at least $1-\delta$, for every successfully selected
function $h_j$, at least $k/2$ groups contain $h_j$ in their lists
when it is selected and store full samples of $n$ examples from every
full mistake batch processed during the run.

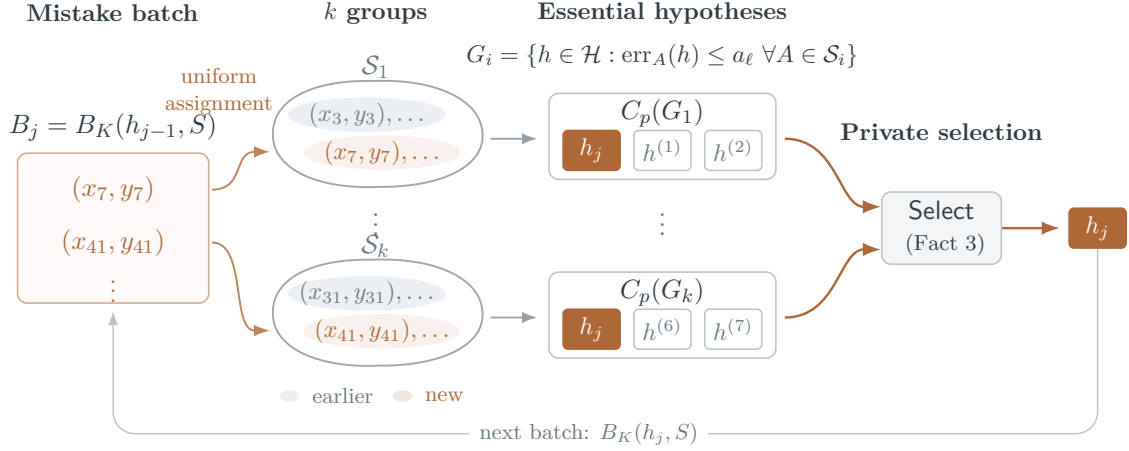
\begin{figure}[!tbp]
\centering
\definecolor{figAmber}{HTML}{AE6736}
\definecolor{figInk}{HTML}{2F3B43}
\definecolor{figSlate}{HTML}{697982}
\definecolor{figLine}{HTML}{BBC4CA}
\definecolor{figOld}{HTML}{ECF0F2}
\resizebox{0.92\linewidth}{!}{%
\begin{tikzpicture}[x=1cm,y=1cm,font=\small,text=figInk,
  line width=0.65pt,>=Latex,
  heading/.style={font=\small\bfseries,align=center},
  neutral/.style={draw=figLine,fill=white,rounded corners=2pt,
    minimum width=0.83cm,minimum height=0.58cm,
    inner sep=2pt,text=figSlate,font=\small},
  shared/.style={draw=figAmber,fill=figAmber,rounded corners=2pt,
    minimum width=0.83cm,minimum height=0.58cm,
    inner sep=2pt,text=white,font=\small},
  edge/.style={->,draw=figSlate!70,line width=0.75pt},
  support/.style={->,draw=figAmber,line width=1.0pt}]
  \path[use as bounding box] (-0.10,-0.35) rectangle (16.60,6.29);

  \node[heading] at (1.42,6.00) {Mistake batch};
  \node[heading] at (5.25,6.00) {$k$ groups};
  \node[heading] at (9.42,6.00) {Essential hypotheses};
  \node[font=\small] at (9.42,5.38)
    {$G_i=\bc{h\in\cH:\operatorname{err}_A(h)\le a_\ell
      \ \forall A\in\cS_i}$};

  \draw[figAmber!55,fill=figAmber!3,rounded corners=4pt]
    (0.04,1.78) rectangle (2.82,3.96);
  \node[font=\normalsize] at (1.43,4.34) {$B_j=B_K(h_{j-1},S)$};
  \foreach \y/\r in {3.43/7,2.66/41}
    \node[text=figAmber,font=\normalsize] at (1.43,\y)
      {$(x_{\r},y_{\r})$};
  \node[text=figAmber] at (1.43,2.06) {$\vdots$};
  \node[text=figAmber,font=\footnotesize,align=center] at (2.97,4.86)
    {uniform\\assignment};

  \draw[->,figAmber!80,line width=0.8pt]
    (2.87,3.43) .. controls (3.35,3.43) and (3.17,3.98) .. (3.71,3.98);
  \draw[->,figAmber!80,line width=0.8pt]
    (2.87,2.66) .. controls (3.38,2.66) and (3.14,1.38) .. (3.71,1.38);

  \foreach \y/\i/\r/\s in {4.25/1/7/3,1.65/k/41/31} {
    \node[font=\normalsize,text=figSlate] at (5.25,\y+0.96)
      {$\cS_{\i}$};
    \draw[draw=figSlate!65,fill=white,line width=0.8pt]
      (3.76,\y-0.10)
      .. controls (3.68,\y+0.47) and (4.31,\y+0.78) .. (5.20,\y+0.76)
      .. controls (6.15,\y+0.80) and (6.77,\y+0.47) .. (6.80,\y+0.02)
      .. controls (6.86,\y-0.47) and (6.22,\y-0.79) .. (5.29,\y-0.78)
      .. controls (4.40,\y-0.82) and (3.81,\y-0.52) .. cycle;
    \fill[figOld] (5.12,\y+0.27) ellipse [x radius=1.15,y radius=0.25];
    \node[text=figSlate] at (5.12,\y+0.27)
      {$(x_{\s},y_{\s}),\ldots$};
    \fill[figAmber!9] (5.34,\y-0.28) ellipse [x radius=1.14,y radius=0.25];
    \node[text=figAmber] at (5.34,\y-0.28)
      {$(x_{\r},y_{\r}),\ldots$};
    \draw[edge] (6.92,\y-0.08) -- (7.60,\y-0.08);

    \draw[draw=figLine,fill=white,rounded corners=4pt]
      (7.77,\y-0.67) rectangle (11.07,\y+0.58);
    \node[font=\normalsize] at (9.42,\y+0.29) {$C_p(G_{\i})$};
  }
  \node[text=figSlate,font=\large] at (5.25,3.03) {$\vdots$};
  \node[text=figSlate,font=\large] at (9.42,3.03) {$\vdots$};
  \fill[figOld] (4.00,0.43) ellipse [x radius=0.14,y radius=0.085];
  \node[anchor=west,font=\footnotesize,text=figSlate] at (4.19,0.43) {earlier};
  \fill[figAmber!16] (5.64,0.43) ellipse [x radius=0.14,y radius=0.085];
  \node[anchor=west,font=\footnotesize,text=figAmber] at (5.83,0.43) {new};

  \node[shared] (sharedOne) at (8.39,4.00) {$h_j$};
  \node[neutral] at (9.42,4.00) {$h^{(1)}$};
  \node[neutral] at (10.45,4.00) {$h^{(2)}$};
  \node[shared] (sharedLast) at (8.39,1.40) {$h_j$};
  \node[neutral] at (9.42,1.40) {$h^{(6)}$};
  \node[neutral] at (10.45,1.40) {$h^{(7)}$};

  \node[draw=figLine,fill=figOld!65,rounded corners=4pt,
    minimum width=1.74cm,minimum height=1.05cm,align=center,
    font=\normalsize] (select) at (13.48,2.87)
      {$\mathsf{Select}$\\[-1pt]{\footnotesize~(\Cref{fact:candidate-selection})}};
  \draw[support]
    (11.20,4.17) .. controls (11.97,4.17) and (11.99,3.18)
      .. ([yshift=0.31cm]select.west);
  \draw[support]
    (11.20,1.57) .. controls (11.97,1.57) and (11.99,2.56)
      .. ([yshift=-0.31cm]select.west);
  \node[heading] at (13.48,4.28) {Private selection};
  \draw[support] (select.east) -- (15.20,2.87);
  \node[shared] (output) at (15.75,2.87) {$h_j$};

  \draw[->,draw=figLine,line width=0.6pt,rounded corners=9pt]
    (output.south) -- (15.75,-0.12) -- (1.43,-0.12) -- (1.43,1.66);
  \node[fill=white,inner sep=4pt,text=figSlate,font=\footnotesize]
    at (8.35,-0.12) {next batch: $B_K(h_j,S)$};
\end{tikzpicture}%
} \caption{\footnotesize $\cS_i$ collects group $i$'s retained
samples.  Each mistake batch contributes one coloured oval~(denoted
\(A\)) of \(\leq n\) points.  $\operatorname{err}_A(h)$ is the mistake
count on $A$, divided by $n$.  The list $C_p(G_i)$ contains the
essential hypotheses of $G_i$.}
\label{fig:private-candidates}
\end{figure}

Condition on this event and suppose, towards a contradiction, that
every candidate makes at least $K$ mistakes. Every batch is then full.
On the same event, \Cref{lem:candidate-selection-completion} ensures that
the algorithm makes $N$ successful selections whenever every mistake
batch is full. Thus, under our assumption, it produces $h_1,\ldots,h_N$.
We obtain a contradiction by counting the distinct
pairs $(i,h)$ for which $h$ appears in group $i$'s list at some point
during the run. Denote their number by $P$.

\noindent\textbf{Upper bound on $P$.}
Let $s$ be the uniform upper bound on candidate-list sizes fixed by the
algorithm's parameters. The bound $|C_p(G)|\le p^d2^{d^2}$, together with
the parameter choices before Algorithm~\ref{alg:private-candidates}, gives $s=2^{\bigO{d^2}}$. Consider one
group and follow its candidate list as the algorithm runs. Using the
properties of essential hypotheses and decomposition dimension,
\Cref{lem:candidate-list-monotonicity} shows that the group's decomposition
dimension cannot increase and, whenever it stays unchanged, the new
candidate list is a subset of the previous list. Therefore, all functions
appearing during a period of constant dimension already belong to its first
list, and there can be at most $s$ of them. New functions can enter only
when the decomposition dimension decreases. Since this dimension starts at
most $d$ and never becomes negative, there are at most $d$ decreases, and
hence at most $d+1$ such periods. Therefore, at most $(d+1)s$ distinct
functions can appear in this group's lists over the entire run. Summing over
the $k$ groups gives $P\le k(d+1)s$.

\noindent\textbf{Lower bound on $P$.}
Consider one of the first $N-1$ selected functions, say $h_j$.
The return arrow in Figure~\ref{fig:private-candidates} represents
its next batch, $B_K(h_j,S)$, consisting entirely of examples on
which $h_j$ is wrong. As established
above, at least $k/2$ groups whose lists contained $h_j$ when it was
selected store a full sample of $n$ examples from this batch. This sample
prevents $h_j$ from ever appearing in any of those groups' lists again.
Indeed, suppose that $h_j$ later belonged to $C_p(G_i)$ for one such group.
Since $p$ is chosen to be at least $n$, the stated property of essential
hypotheses gives a function $g\in G_i$ agreeing with $h_j$ on every instance
in the stored sample. The function $g$ would therefore make $n$ mistakes on
that sample, contradicting the requirement that it make at most $a_\ell n<n$
mistakes. Because the sample is retained and the threshold can only
decrease, this argument applies at every later stage. Each of the first
$N-1$ selections thus removes at least $k/2$ pairs counted in $P$, and no
pair can be removed again at a later selection. Hence $P\ge(N-1)k/2$.

Combining the two bounds gives
\[
  (N-1)\frac{k}{2}\le P\le k(d+1)s,
\]
which contradicts the choice $N=2(d+1)s+2$. Thus, with probability at
least $1-\delta$, some candidate makes fewer than $K$ mistakes. The
privacy guarantee along with specific parameter choices are proved in
\Cref{app:private-candidates}.

\noindent\textbf{Step 2: using the candidates online.}
The candidates can be computed as their preceding mistake batches
become full. However, the private candidate sequence does not
protect the rounds when this happens. We therefore need a
prediction rule in which a candidate cannot affect the output
before the batches needed to compute it become full.

We use a private replacement algorithm $\mathsf{Rep}$
(\Cref{def:online-replacement}), based on the Mirror mechanism
of~\cite{CLNSS24}. Given two sequences of binary predictions $U,V$,
it returns their common prediction whenever they agree. On
disagreement rounds, it initially follows $U$. We test privately
whether $V$ is correct sufficiently often on the initial
disagreements, and use Mirror to delay when an accepted candidate
may affect the predictions.

Write $W=\mathsf{Rep}(U,V,D)$, where $D$ controls the delay.
\Cref{lem:online-replacement} gives the following two properties
with probability one:

\begin{enumerate}[label=(\roman*),leftmargin=*]
\item If $V$ makes less than $K$ mistakes on $S_{1:r}$, then \(W\)
makes at most $M(W;S_{1:r})\le K+2D$.
\item If $U$ made at most $D/4$ mistakes before round $t$, then no
replacement occurs and $W_t=U_t$.
\end{enumerate}

The first property bounds the mistakes of the combined predictor
when the new candidate is good. The second prevents a new candidate
from changing the output while the preceding predictor has made
few mistakes.

Choose $D_0>4K$ and set $D_j=9^jD_0$. First suppose that all
candidates $h_1,\ldots,h_N$ are available, and define
$\cA_0=h_0$ and
$\cA_j=\mathsf{Rep}(\cA_{j-1},h_j,D_j)$ for $1\le j\le N$.
The difficulty is that computing the next candidate depends on
mistakes of $h_j$, whereas replacement depends on mistakes of
$\cA_j$. The first property connects these two mistake counts.
The increasing delays then allow us to apply the second property
and show that later candidates cannot change the predictions
before they can be computed.

\noindent\textbf{Computing the predictions online.}~Suppose $0\le m<N$
where $h_m$ has fewer than $K$ mistakes before round $t$. The first
property above gives $M(\cA_m;S_{1:s-1})\le K+2D_m$ for every $s\le
t$.  Since $D_m\ge D_0\ge4K$, we have $K+2D_m\le
D_m/4+2D_m=D_{m+1}/4$.  The second property therefore shows that
$\cA_{m+1}$ has the same predictions as $\cA_m$ on every round through
$t$.  

Their mistake counts are therefore equal until \(t\). If $m+1<N$,
then, for every $s\le t$, $M(\cA_{m+1};S_{1:s-1})=M(\cA_m;S_{1:s-1})
\le D_{m+1}/4\le D_{m+2}/4$.  Then using the second property again
shows that $\cA_{m+2}$ has the same predictions as $\cA_{m+1}$ on
every round through $t$.  Repeating this argument shows that, whenever
$M(h_m;S_{1:t-1})<K$, the predictors $\cA_m,\ldots,\cA_N$ agree on
every round through $t$. \Cref{alg:online-candidates} uses this
agreement property to implement $\cA_N$ online. It computes successive
candidates as their preceding mistake batches become full and predicts
using $\cA_m$, where $h_m$ is the last candidate computed. We denote
this online predictor by $\cA$.

\noindent\textbf{Mistake bound for \(\cA\)} By
\Cref{lem:candidate-sequence}, with probability at least $1-\delta$,
some candidate $h_j$ makes fewer than $K$ mistakes on the whole
stream. Using the first property, we then have that $\cA_j$ makes at
most $K+2D_j$ mistakes. Since $M(h_j;S)<K$, using the above argument
we have that $\cA_j,\ldots,\cA_N$ predict identically on the whole
stream. The online predictor $\cA$ therefore makes at most
$B:=K+2\cdot9^ND_0$ mistakes. Substituting $N\le2^{\bigO{d^2}}$ and
$K\le\frac{2^{\bigO{d^2}}}{\epsilon}\log\frac{2}{\epsilon\delta}$ from
\Cref{lem:candidate-sequence} into $B$, and using the replacement
mechanism's parameter choice, gives
\begin{equation}\label{eq:prediction-high-probability-bound}
  \bP\bs{M(\cA;S)\le B}\ge1-\delta,
  \qquad
  B\le\frac{2^{2^{\bigO{d^2}}}}{\epsilon^2}
       \log^2\frac{2}{\epsilon\delta}.
\end{equation}
The additional factor $\epsilon^{-1}\log(2/(\epsilon\delta))$ arises
from the choice of $D_0$ used to obtain the replacement mechanism's
privacy and delay guarantees.

\noindent\textbf{Privacy guarantee of \(\cA\)}
We first analyse the predictor $\cA_N$ defined using the entire
candidate sequence. \Cref{lem:candidate-sequence} gives an
$(\epsilon/4,\delta/8)$-differentially private candidate sequence.
Each of the $N$ replacement mechanisms is
$(\epsilon/(4N),\delta/(8N))$-jointly differentially private on its
input stream $(U_t,V_t,y_t)_{t=1}^T$. The composition argument
in \Cref{lem:online-candidates-privacy} therefore gives
$(\epsilon/2,\delta/4)$-joint differential privacy for $\cA_N$.
The online predictor $\cA$ has exactly the same predictions as
$\cA_N$ on every stream.
Thus the same privacy guarantee holds for $\cA$, independently
of the event used to prove the mistake bound.

\noindent\textbf{Step 3: an expected mistake bound independent of $T$.}
The high-probability guarantee from Step~2 leaves an event of probability
at most $\delta$ on which $\cA$ may make as many as $T$ mistakes. It
therefore only gives $\bE M(\cA;S)\le B+\delta T$. We remove this
dependence on $T$ by combining $\cA$ with another private predictor
whose expected mistakes are bounded, albeit by a worse quantity.

This other predictor is the Name and Shame SOA algorithm
of~\cite{DSS24}, which we denote by $\cF$ and run
independently of $\cA$. It predicts using SOA on the examples retained
so far and, after each prediction, retains the new labelled example
with probability $\delta/4$. It is easy to see
(\Cref{fact:candidate-subsampled-soa}) that $\cF$ is
$(0,\delta/4)$-jointly DP and satisfies $\bE M(\cF;S)\le4d/\delta$.

Our second technical ingredient is a private comparison rule that
combines the two predictors to get a best-of-both-worlds type mistake
bound. The algorithm (\Cref{alg:prediction-expectation}) returns their
common prediction when they agree; otherwise, it favours the predictor
that was correct on more of the previous disagreement rounds. Let
$\cC=\mathsf{Combine}_D(\cA,\cF)$, where
$D=\ceil{C\epsilon^{-2}\log(2/\delta)}$ for a sufficiently large
absolute constant $C$. \Cref{lem:prediction-combination} shows
$\bE M(\cC;S)\le3\bE\min\bc{M(\cA;S),M(\cF;S)}+2D$.

Let $E=\bc{M(\cA;S)\le B}$. On $E$, the minimum in this bound is at
most $B$; on $E^c$, it is at most $M(\cF;S)$. Since the stream is fixed
and $\cF$ runs independently of $\cA$, the event $E$ is independent of
$M(\cF;S)$. Therefore,
\[
\begin{aligned}
  \bE M(\cC;S)
  &\le3B+3\bP(E^c)\,\bE M(\cF;S)+2D\le3B+3\delta\cdot\frac{4d}{\delta}+2D
   =3B+12d+2D.
\end{aligned}
\]
Substituting the bound on $B$ from Step~2 and our choice of $D$ gives
$\bE M(\cC;S)\le\frac{2^{2^{c(d+1)^2}}}{\epsilon^2}
\log^2\frac{2}{\epsilon\delta}$ for an absolute constant $c$. Finally,
\Cref{lem:prediction-expectation} in
Section~\ref{sec:prediction-expectation} proves that $\cC$ is
$(\epsilon,\delta)$-jointly differentially private, completing the proof.
\end{proof}

\section{Open problems and discussion}
\label{sec:conclusion}
In this work, we proved two main results for Differentially Private
Online Learning. The first result~\Cref{thm:dimension} shows a lower
bound of \(\Om{d\log^{2/3}T}\) for private online learning of a
hypothesis class \(\cH\) with \(d=\Ldim\br{\cH}\ge2\), for fixed
\(0<\epsilon\le1\), \(\delta=\bigO{1/\log T}\), and sufficiently
large \(T\). The primary
ingredient of this proof is a similar lower bound for the \(\Point\)
class, which was left open in previous works~\cite{DSS24,LWY24}.
This leaves several questions open. At
\(\delta=\Theta(1/\log T)\), the first is to close the gap between the
\(O(\log T)\) upper bound~\cite{LWY25} and our
\(\Omega((\log T)^{2/3})\) lower bound for fixed \(d,\epsilon\). Our lower bound construction is agnostic of the
algorithm of the online learner unlike those of previous
works~\cite{DSS24,LWY24}. One interesting direction is to see if
the lower bound can be improved by using a construction that is
adaptive to the algorithm of the online learner.

Our second result~\Cref{thm:online-prediction-expectation} shows an
upper bound that is independent of \(T\) in the private prediction
setting. In this work, we did not optimise the dependence on the
Littlestone dimension. It is perhaps possible to achieve a dependence
that is polynomial in the Littlestone dimension. Extension to
the agnostic setting as well as learning under stochastic adversaries
under smooth distributions is also an interesting direction for future
work.

\section{Proof of the lower bound for POINT}
\label{app:auxiliary-results}
\label{app:point-lower-bound}

\PointLowerBound*

We use the notation of Section~\ref{sec:main-results}: $r$ indexes the
learner's rounds, $t$ counts completed stages of the adversary, and
$Y_r=h_r(1)$. Write $Y_{\le r}=(Y_1,\ldots,Y_r)$.~\Cref{alg:adversary}
below gives the construction when $T$ is a power of three but the
result easily extends to arbitrary $T$. Apply the construction to the
first $3^L$ rounds, where $L=\floor{\log_3T}$, and leave the remaining
rounds dummy.

\begin{algorithm}[H]
\caption{Adversary for $\Point$}\label{alg:adversary}
\small
\begin{algorithmic}[1]
\Require $T\ge3$ a power of three, $p\in(0,1]$.
\State Draw $B$ uniformly from $\bc{0,1}$; set $S_r\gets\br{3,0}$ for all $r$;set $A_0\gets(0,T]$ and $L\gets\log_3T$.
\For{$t=0,\ldots,L-1$}
    \State Partition $A_t$ into three equal consecutive subintervals
    $A_t^1,A_t^2,A_t^3$.
  \State Draw $J_{t+1}$ uniformly from $\bc{1,2,3}$.
  \If{$J_{t+1}=2$}
    \State Draw a round $s$ uniformly from $A_t^1$.
    \State With probability $p$, replace $S_s$ by $\br{1,B}$.
  \EndIf
  \State Set $A_{t+1}\gets A_t^{J_{t+1}}$.
\EndFor
\State \Return $\br{S_1,\ldots,S_T}$.
\end{algorithmic}
\end{algorithm}

\MistakeContribution*

\begin{proof}[Proof of~\Cref{lem:mistakes}]
Fix $J_{\le t}$, and condition on an informative example being
inserted at a round $s\in A_t^1$ at stage $t+1$. The predictions
through round $s$ have the distribution used to define $V_t(s)$: the
learner predicts before receiving the example at $s$, and every
example inserted at a later stage occurs after $s$. Write
$m_s=\bE[2B-1\mid Y_{\le s}]$. Given these predictions, the
probability of incorrectly predicting $B$ is at least $(1-|m_s|)/2$.
Consequently,
\[
  \bP(Y_s\ne B)
  \ge\frac{1-\bE|m_s|}{2}
  \ge\frac{1-\sqrt{V_t(s)}}{2}
  \ge\frac{1-\sqrt{I_t}}{2}.
\]
The last two inequalities use Cauchy--Schwarz and $V_t(s)\le
V_t(b_t)\le I_t$.  An informative example is inserted at this stage
with probability $p/3$. Averaging over its location and over $J_{\le
t}$, and using concavity of the square root, gives the desired bound:
\[
  \frac p6\br{1-\bE\sqrt{I_t}}
  \ge\frac p6\br{1-\sqrt{\bE I_t}}.
\]
\end{proof}

\VarianceRecursion*

\begin{proof}[Proof of~\Cref{lem:variance-recursion}]
Fix $J_{\le t}$ and write \[ A_t=(a_t,b_t],\enskip
  A_t^1=(a_t,c_1],\enskip A_t^2=(c_1,c_2],\enskip A_t^3=(c_2,b_t].  \]
  Recall the distributions $P_\varnothing^b$ and $P_s^b$ defined
  in~\Cref{sec:main-results}. Let $Q$ and $Q_s$ be their joint
  distributions with the fair bit $B$, so
  $Q(B=b,Y=y)=\tfrac12P_\varnothing^b(y)$ and
  $Q_s(B=b,Y=y)=\tfrac12P_s^b(y)$.  For $j\in\bc{1,2,3}$, let
  $R_j=|A_t^j|^{-1}\sum_{s\in A_t^j}Q_s$ be the joint distribution
  after inserting an informative example at a uniformly chosen round
  in the $j$th third.

We first bound the error introduced by clipping the likelihood ratios.
Conditional on \(B\) and
the earlier insertions, adding an example at \(s\) changes one
labelled example. Averaging the privacy inequalities over this
randomness shows that they also hold for \(Q_s\) and \(Q\).
Clipping their ratio to
\([e^{-\epsilon},e^\epsilon]\) therefore gives
\begin{equation}\label{eq:app-clipping}
\begin{aligned}
  \|Q_s-Q\ell_s^B\|_1&\le2\delta,\\
  \|R_j-Q\overline\ell_{A_t^j}^B\|_1&\le2\delta.
\end{aligned}
\end{equation}
Here \(Q\overline\ell_{A_t^j}^B\) denotes the measure with density
\(\overline\ell_{A_t^j}^B\) relative to \(Q\).  Indeed, the mass above
\(e^\epsilon Q\) is at most \(\delta\), and the deficit below
\(e^{-\epsilon}Q\) is at most \(e^{-\epsilon}\delta\). Throughout,
using~\(\epsilon=1/100\) gives \(|\ell_s^B-1|\le1/8\) and
\(e^\epsilon\le5/4\).

\noindent\textbf{Bounding the increase in $I_t$.} Let $V_t^+(r)$ be
$V_t(r)$ computed under $P=(1-p)Q+pR_1$, the distribution when the
middle third is chosen. We show that, for \(0\le u\le v\le c_2\),

\begin{equation}\label{eq:app-increment-direct}
  V_t^+(v)-V_t^+(u)
  \le 2\br{V_t(v)-V_t(u)}+8p^2D_t+16p\delta.
\end{equation}

The increment \(V_t^+(v)-V_t^+(u)\) is the reduction in the minimum
expected squared error for estimating \(2B-1\) when we observe
predictions through round \(v\), rather than only through round \(u\).

Write
\[
  Z=2B-1,\qquad
  m_r=\bE_Q[Z\mid Y_{\le r}],\qquad
  f=\bE_P[Z\mid Y_{\le v}],
\]
choosing values in $[-1,1]$ on histories of probability zero.
Let $H=(Z-m_u)^2-(Z-f)^2$ be the improvement in squared error
when we replace $m_u$ by $f$. Since $m_u$ is an available estimate
at round $u$, while $f$ is the best estimate at round $v$ under $P$,
\[
  V_t^+(v)-V_t^+(u)\le\bE_P H.
\]

We compare this improvement with the one under $Q$.
An insertion in the last third cannot affect predictions through
round $v\le c_2$. Since $H$ depends only on $B$ and these predictions,
$\bE_{R_3}H=\bE_QH$. Also, $|H|\le4$, so applying
\eqref{eq:app-clipping} to $R_1$ and $R_3$ gives
\[
\begin{aligned}
  \bE_P H
  &=\bE_QH+p\br{\bE_{R_1}H-\bE_{R_3}H}\\
  &\le\bE_QH+
    p\bE_Q\bs{
      \br{\overline\ell_{A_t^1}^B-\overline\ell_{A_t^3}^B}H}
    +16p\delta.
\end{aligned}
\]
Each clipping bound contributes at most $p\cdot4\cdot2\delta$.

Moreover, $|H|\le4|f-m_u|$.
Using $4p|ab|\le8p^2a^2+b^2/2$ and the definition of $D_t$, we get
\[
  p\bE_Q\bs{
    \br{\overline\ell_{A_t^1}^B-\overline\ell_{A_t^3}^B}H}
  \le8p^2D_t+\frac12\bE_Q\bs{(f-m_u)^2}.
\]
The last term is cancelled by the squared-error comparison under $Q$.
Indeed, conditioning on $Y_{\le v}$ and using
$2ab\le2a^2+b^2/2$ gives
\[
\begin{aligned}
  \bE_QH
  &=\bE_Q\bs{2(m_v-m_u)(f-m_u)-(f-m_u)^2}\\
  &\le2\bE_Q\bs{(m_v-m_u)^2}
       -\frac12\bE_Q\bs{(f-m_u)^2}.
\end{aligned}
\]
Combining the bounds therefore yields
\[
  V_t^+(v)-V_t^+(u)
  \le2\bE_Q\bs{(m_v-m_u)^2}+8p^2D_t+16p\delta.
\]

Since
\(m_r=\bE_Q[Z\mid Y_{\le r}]\) and \(m_u\) is determined by \(Y_{\le
u}\), for \(u\le v\) we have
\[
  \bE_Q[m_vm_u]=\bE_Q[Zm_u]=\bE_Q[m_u^2].
\]
The first equality conditions on \(Y_{\le v}\), and the second
on \(Y_{\le u}\). Expanding the square therefore gives
\[
\begin{aligned}
  \bE_Q[(m_v-m_u)^2]
  &=\bE_Q[m_v^2]+\bE_Q[m_u^2]-2\bE_Q[m_vm_u]=\bE_Q[m_v^2]-\bE_Q[m_u^2]=V_t(v)-V_t(u).
\end{aligned}
\] 

This proves~\eqref{eq:app-increment-direct}.

Choosing the first or last third leaves the stream distribution
unchanged. In the middle case, the insertion occurs after round
$a_t$, so $V_t^+(a_t)=V_t(a_t)$, and
\[
  I_{t+1}=V_t(a_t)
    +\br{V_t^+(c_1)-V_t^+(a_t)}
    +2\br{V_t^+(c_2)-V_t^+(c_1)}.
\]
Apply~\eqref{eq:app-increment-direct} to the increments across the
first and middle thirds. Averaging this bound with
$I_{t+1}=2V_t(c_1)-V_t(a_t)$ in the first case and
$I_{t+1}=2V_t(b_t)-V_t(c_2)$ in the last case gives
\[
\begin{aligned}
  \bE[I_{t+1}\mid J_{\le t}]
  &\le V_t(a_t)+\frac53\br{V_t(c_2)-V_t(a_t)}
       +\frac23\br{V_t(b_t)-V_t(c_2)}+8p^2D_t+16p\delta\\
  &\le I_t+8p^2D_t+16p\delta\le I_t+20p^2D_t+20p\delta.
\end{aligned}
\]
The second inequality uses the nonnegativity of both increments
and the definition of $I_t$.

\noindent\textbf{Bounding the change in $W_t$.} For nonempty
$S\subseteq A_t$, write $W(S)=\bE_Q[\Var_{s\in S}(\ell_s^B)]$, so that
$W(A_t)=W_t$. Let $\overline\ell_{A_t}^B$ denote the average over the
whole interval.  Variance decomposition across the three thirds gives
\[
  W_t
  =\frac13\sum_{j=1}^3W(A_t^j)
   +\frac13\sum_{j=1}^3\bE_Q(\overline\ell_{A_t^j}^B-\overline\ell_{A_t}^B)^2
  \ge\frac13\sum_{j=1}^3W(A_t^j)+\frac{D_t}{6}.
\]
The stream distribution stays the same when the first or last third
is chosen. When the middle third is chosen, the possible insertion
changes it to $P$. Let $W^+(A_t^2)$ be the corresponding variance.
We show that
\begin{equation}\label{eq:app-mixture}
  W^+(A_t^2)\le W(A_t^2)+p+\delta.
\end{equation}

For \(s\in A_t^2\), let \(R_{1,s}\) be the distribution obtained
from \(R_1\) by additionally inserting \((1,B)\) at \(s\), and
let \(\ell_s^\star\) be its ratio relative to \(R_1\),
clipped to \([e^{-\epsilon},e^\epsilon]\).
The comparison distribution relative to \(P\) is
\(P_s^+=(1-p)Q_s+pR_{1,s}\). Clipping each component ratio and then
taking their mixture gives the ratio
\[
  u_s
  =\frac{(1-p)Q\ell_s^B+pR_1\ell_s^\star}{P}
  =(1-\theta)\ell_s^B+\theta\ell_s^\star,
  \]
where \(  \theta=\frac{pR_1}{P}\). The weight \(\theta\) depends on $B$ and the released predictions,
but not on the candidate insertion round \(s\).
Convexity of variance therefore gives
\[
\begin{aligned}
  \bE_P\bs{\Var_{s\in A_t^2}(u_s)}
  &\le
    (1-p)W(A_t^2)
    +p\,\bE_{R_1}
       \bs{\Var_{s\in A_t^2}(\ell_s^\star)}\le W(A_t^2)+p.
\end{aligned}
\]
The last inequality follows from the bounds on the clipped ratios. Let \(\ell_s^+\) be the actual clipped ratio \(P_s^+/P\).
Applying~\eqref{eq:app-clipping} to both component comparisons
gives \(\|P_s^+-Pu_s\|_1\le2\delta\).
Since \(u_s\) already lies in the clipping interval, projecting
\(P_s^+/P\) onto that interval cannot increase its distance
from \(u_s\), so \(\bE_P|\ell_s^+-u_s|\le2\delta\).
For any two families \(u_s,v_s\in[7/8,9/8]\),
\[
  \abs{\Var_s(u_s)
        -\Var_s(v_s)}
  \le\frac12\,\bE_s|u_s-v_s|.
\]
Applying this inequality with \(v_s=\ell_s^+\), and averaging
over \(P\), proves~\eqref{eq:app-mixture}. Combining the variance decomposition with~\eqref{eq:app-mixture}
yields
\[
  \bE[W_{t+1}\mid J_{\le t}]
  =\frac{W(A_t^1)+W^+(A_t^2)+W(A_t^3)}3
  \le W_t-\frac{D_t}{6}+\frac{p+\delta}{3}.
\]
\end{proof}

\section{Proof of the lower bound for general Littlestone classes}
\label{app:dimension-extension}

\begin{definition}[Littlestone dimension]\label{def:littlestone}
Consider a complete binary tree whose internal nodes are labelled by
points of $\cX$, with the two edges leaving each node labelled zero
and one. The tree is \emph{shattered} by $\cH$ if every root-to-leaf
path is realised by some $f\in\cH$: at each node on the path, $f$
assigns the label of the edge followed by that path. The depth is
the number of internal nodes on each root-to-leaf path.
The Littlestone dimension $\Ldim\br{\cH}$ is the largest depth of
a tree shattered by $\cH$, or infinity if such depths are unbounded.
\end{definition}

Without privacy, a class of Littlestone dimension $d$ admits a learner
making at most $d$ mistakes on every realisable
stream~\cite{Littlestone88}. In this section, we extend the
informative-mistake bound of~\Cref{thm:main} to general $\epsilon$ and
apply it at successive nodes of a shattered tree.

\begin{lemma}
\label{lem:informative-fraction}
Fix distinct points $x_0,x$, let $N\ge1$, and suppose that
$0\le\delta'\le1/4$. For every
$\br{\epsilon_0,\delta'}$-private online learner, where
$\epsilon_0=1/100$, there is a distribution over streams of length
$N$ consisting only of $\br{x_0,0}$ and $\br{x,B}$, with $B$ a fair bit. If $K$
counts informative examples and $X$ counts mistakes on those
examples, then, writing $\mu=\bE K$,
\begin{equation}\label{eq:informative-fraction}
\begin{gathered}
  a\min\bc{\br{\log\br{2N}}^{2/3},1/\delta'}
  \le\mu\le
  A\min\bc{\br{\log\br{2N}}^{2/3},1/\delta'},
  \\
  \bE X\ge\mu/4,
\end{gathered}
\end{equation}
where $a,A>0$ are absolute constants. 
\end{lemma}

\begin{proof}
Use the construction from~\Cref{sec:main-results}, with
$\br{x_0,0}$ as the dummy example and $\br{x,B}$ as the informative
example. For $N\ge3$, set $L=\floor{\log_3N}$ and apply
Algorithm~\ref{alg:adversary} to the first $3^L$ rounds, leaving the
remaining rounds dummy. Choose
$p=a_0\min\bc{L^{-1/3},1/\br{\delta'L}}$
for a sufficiently small absolute constant $a_0>0$.
The proof of Theorem~\ref{thm:main}
gives $\bE I_t\le1/4$ for every $t<L$. Thus
$\mu=Lp/3$, and Lemma~\ref{lem:mistakes} gives
$\bE X\ge Lp/12=\mu/4$. Comparing $L$ with $\log\br{2N}$ proves
the bounds on $\mu$.

For $N=1,2$, use one informative example with a fair label at
the first round and fill any remaining round with the dummy example.
Then $\mu=1$ and $\bE X=1/2$, giving the same bounds after
adjusting the constants $a,A$.
\end{proof}

We recall two properties of differential privacy under replacement
adjacency.
\begin{fact}[Group privacy]
\label{prop:group-privacy}
Let $\cM$ be $\br{\epsilon,\delta}$-differentially private and let
$q\ge1$ be an integer. If $S,S'$ differ in at most $q$ positions,
then, for every measurable output event $E$,
\[
  \bP\bs{\cM\br{S}\in E}
  \le e^{q\epsilon}\bP\bs{\cM\br{S'}\in E}
     +\delta\sum_{j=0}^{q-1}e^{j\epsilon}.
\]
\end{fact}

\begin{proposition}
\label{prop:dummy-replacement}
Fix $z_\star\in\cX\times\bc{0,1}$ and $\gamma\in\bs{0,1}$.
Given an $\br{\epsilon,\delta}$-private online learner, independently
retain each input example with probability $\gamma$ and otherwise
feed $z_\star$ to the learner, without removing the round. Releasing
the learner's predictors gives a
$\br{\log\br{1+\gamma\br{e^\epsilon-1}},\gamma\delta}$-private
online learner.
\end{proposition}

\begin{proof}
Fix neighbouring streams and all retention decisions except at the
changed round. For an output event $E$, let $a,b,u$ be its
probabilities when that round contains the first stream's example, the
second stream's example, or $z_\star$, respectively. Privacy gives
$a\le e^\epsilon\min\bc{b,u}+\delta$. Since $\min\bc{b,u}\le b$ and
$\min\bc{b,u}\le\br{1-\gamma}u+\gamma b$, we obtain
\[
  \br{1-\gamma}u+\gamma a
  \le\bs{1+\gamma\br{e^\epsilon-1}}
        \bs{\br{1-\gamma}u+\gamma b}+\gamma\delta.
\]
Averaging over the remaining retention decisions proves the claim.
\end{proof}

\begin{lemma}
\label{lem:one-bit-general-privacy}
Fix distinct points $x_0,x$, let $n\ge1$, and suppose that
$0<\epsilon\le1$ and
$0\le\delta\le\min\bc{\epsilon^2,1/4}$.
For every $\br{\epsilon,\delta}$-private online learner, there are
a bit $b$ and a deterministic stream of length $n$ consisting only
of $\br{x_0,0}$ and $\br{x,b}$ such that the number $X$ of
mistakes on informative examples satisfies
\begin{equation}\label{eq:one-bit-extension}
  \bE X\ge c\min\bc{
    n,\;
    \frac1\epsilon\bs{\log\br{2+\epsilon n}}^{2/3},\;
    \frac1\delta
  }.
\end{equation}
In particular, if $n\ge\max\bc{2,\epsilon^{-2}}$ and
$\delta\le\epsilon/\br{\log n}^{2/3}$, then the same stream
satisfies
\begin{equation}\label{eq:one-bit-logarithmic}
  \bE X\ge\frac{c'}\epsilon\br{\log n}^{2/3}.
\end{equation}
Here $c,c'>0$ are absolute constants.
\end{lemma}

\begin{proof}
Write $\epsilon_0=1/100$ and first suppose that
$\epsilon\le\epsilon_0$ and $n\ge1/\epsilon$. We replace each
auxiliary example with $q$ copies. Starting these copies at a
uniformly chosen position lets us compare the mistakes on later copies
with the auxiliary mistake at the first copy. Choose
\(q=\floor{\epsilon_0/\epsilon}, w=16q, N=\floor{\frac n{2w}}.  \)
Each round of the auxiliary learner uses a block of $2w$ rounds of the
original learner. Independently in each block, choose $U$ uniformly
from $\bc{1,\ldots,w}$ and feed dummy examples before position $U$.
Release the original learner's predictor at $U$ as the auxiliary
predictor. Then receive the auxiliary example and place $q$ copies at
positions $U,\ldots,U+q-1$. All other positions in the block and any
rounds after the $N$ blocks are dummy.

For fixed offsets, replacing one auxiliary example changes at most
$q$ original examples. Group privacy gives
$\br{\epsilon_0,\delta_q}$-privacy, where
\[
  \delta_q=\delta\sum_{j=0}^{q-1}e^{j\epsilon}
  \le e^{\epsilon_0}q\delta
  \le e^{\epsilon_0}\epsilon_0^2<10^{-3}.
\]
The same guarantee holds after averaging over the offsets.
Apply~\Cref{lem:informative-fraction} with $\delta'=\delta_q$, and let
$\mu$ be its expected informative count.

We next compare mistakes on the copies with the auxiliary mistake.
Fix the preceding input and an informative label, and let $\alpha$ be
the auxiliary mistake probability. Since the prediction at $U$
precedes the first copy, $\alpha$ is also the mistake probability at a
uniformly chosen position $U$ when the whole block is dummy. For $0\le
j<q$, the windows $\bc{1,\ldots,w}$ and $\bc{1+j,\ldots,w+j}$ differ
by $j$ positions at each end.  Since each error probability lies in
$[0,1]$, shifting the window reduces its average by at most $j/w$.
Group privacy therefore bounds the mistake probability on the copy at
$U+j$, averaged over $U$, from below by
\[
  e^{-q\epsilon}\br{\alpha-j/w}-\delta_q.
\]
Summing over the copies and informative auxiliary examples, and using
the auxiliary mistake bound $\mu/4$, gives
\[
  \bE X
  \ge q\mu\bs{
    e^{-\epsilon_0}\br{\frac14-\frac1{32}}-\delta_q}
  \ge \frac{q\mu}{5},
\]
where we used $(q-1)/(2w)\le1/32$. Since $q=\Theta(1/\epsilon)$ and
$N=\Theta(\epsilon n)$, \eqref{eq:informative-fraction} and
$\delta_q\le e^{\epsilon_0}q\delta$ give \eqref{eq:one-bit-extension}.
For $n<1/\epsilon$, repeat $(x,B)$ for all $n$ rounds, where $B$ is a
fair bit. Comparing with the all-dummy stream by group privacy, using
$n\epsilon<1$ and $n\delta\le\epsilon_0$, gives
\[
  \bE X\ge n\br{e^{-1}/2-e\epsilon_0}\ge n/8.
\]
Averaging over $B$ gives a deterministic stream. For
$\epsilon_0\le\epsilon\le1$, apply~\Cref{prop:dummy-replacement} with
$r=(e^{\epsilon_0}-1)/(e^\epsilon-1)$.  The auxiliary learner is
$\br{\epsilon_0,r\delta}$-private.  As the prediction is done before
deciding whether to retain, expected mistakes on retained informative
examples equal $r$ times the auxiliary informative
mistakes.~\Cref{lem:informative-fraction} therefore gives
\[
  \bE X=\Om{\min\bc{
    r\br{\log\br{2n}}^{2/3},1/\delta}}.
\]
Since $r=\Theta(1/\epsilon)$ on this range and
$\log(2+\epsilon n)\le2\log(2n)$, this
proves~\eqref{eq:one-bit-extension}.

All input choices can be made before running the original learner,
and every resulting stream uses only the dummy example and
informative examples with one fixed label. Averaging over these
choices therefore gives a deterministic stream with the same
lower bound on expected informative mistakes.

Finally, suppose that $n\ge\max\bc{2,\epsilon^{-2}}$ and
$\delta\le\epsilon/\br{\log n}^{2/3}$. Then
$\epsilon n\ge\sqrt n$, so
$\log\br{2+\epsilon n}\ge\frac12\log n$.
Also, $\br{\log n}^{2/3}\le\sqrt n\le\epsilon n$ and
$1/\delta\ge\epsilon^{-1}\br{\log n}^{2/3}$. Thus each term in
the minimum in~\eqref{eq:one-bit-extension} is at least
$2^{-2/3}\epsilon^{-1}\br{\log n}^{2/3}$, which
proves~\eqref{eq:one-bit-logarithmic}.
\end{proof}

\DimensionLowerBound*

\begin{proof}[Full proof of Theorem~\ref{thm:dimension}]
Take a depth-$d$ tree shattered by $\cH$ and follow the zero (left)
edge at its root $x_0$. The remaining subtree has depth $d-1$, and
is shattered by the hypotheses assigning label zero to $x_0$.
Use $\br{x_0,0}$ as the dummy example throughout. Divide the horizon
into $k=\min\bc{d-1,T}$ phases of length $n=\floor{T/k}$,
leaving any unused rounds dummy.

At the current node $x\ne x_0$, both labels are available.
Fix the input from earlier phases and temporarily fill all later
rounds with dummy examples. Restricting the original learner's
output to the current phase defines an
$\br{\epsilon,\delta}$-private learner on that phase: replacing
one example in the phase replaces only one example in the full
stream. We fix only the input prefix and average over the
learner's randomness.

Apply Lemma~\ref{lem:one-bit-general-privacy} at $x$ to choose a
deterministic hard phase and its label $b$, then follow the edge
labelled $b$. Repeat for all $k$ phases. These choices use the
learner's distribution, so the entire stream is fixed before the
learner runs. The labels follow one path in the shattered tree,
making the stream realisable. Moreover, replacing the temporary
dummy suffix with later phases cannot change earlier predictions.
Adding the expected mistakes over the phases gives
\[
  \bE\bs{M_T}\ge c_2k\min\bc{
    n,\;
    \frac1\epsilon\bs{\log\br{2+\epsilon n}}^{2/3},\;
    \frac1\delta
  }.
\]

If $T<d$, then $k=T$ and $n=1$, giving the required $\Om{T}$
bound. Otherwise, $k=d-1$, so
\[
  \frac d2\le k\le d,
  \qquad \frac T{2k}\le n\le\frac Tk.
\]
In particular, $kn\ge T/2$ and $k/\delta\ge d/\br{2\delta}$.  For the
logarithmic term, $n\ge T/\br{2d}$ gives $\log\br{2+\epsilon
n}\ge\frac12\log\br{2+\epsilon T/d}$.  Comparing the three terms
proves the stated bound.
\end{proof}

\section{Proof of the private prediction bound}
\label{app:online-prediction}
As in~\Cref{sec:private-candidates}, let $\cH$ be nonempty
with $d=\Ldim(\cH)<\infty$, and fix $0<\epsilon\le1$ and
$0<\delta<1/10$.

\PredictionBound*

If $d=0$, a fixed hypothesis in $\cH$ predicts without mistakes.
For $d\ge1$, we give the details of the three-step proof in
Section~\ref{sec:private-candidates}.

\subsection{Constructing a private candidate sequence}
\label{app:private-candidates}

In this subsection, $\epsilon,\delta$ and $\beta$ denote the privacy
parameters and failure probability supplied to the offline algorithm.
For \Cref{lem:candidate-sequence}, substitute
$(\epsilon/4,\delta/8)$ for the privacy parameters and $\beta=\delta$.
We assume $d\ge1$, $0<\epsilon\le1$, and $0<\delta,\beta<1/10$.

Recall that $B_K(h,S)$ consists of the first $K$ examples on which $h$
makes a mistake, or all such examples if there are fewer than $K$.
\[
  B_K(h,S)=\bc{(t,x_t,y_t):h(x_t)\ne y_t,
    \sum_{r=1}^t\mathbf1\bc{h(x_r)\ne y_r}\le K}.
\]
We recall irreducibility, decomposition dimension, and essential
hypotheses from~\cite[Definitions~4.1--4.3]{Lyu26}.
Here, the ambient bound $d=\Ldim(\cH)$ is fixed
and $p\ge1$ is an integer.
For a sequence $Q$ of labelled examples, let $G|_Q$ denote the
functions in $G$ consistent with $Q$. Set $\Ldim(\varnothing)=-1$ and
\[
  \mathrm{SOA}_G(x)=
  \begin{cases}
    0,&\Ldim(G|_{(x,0)})=\Ldim(G),\\
    1,&\text{otherwise}.
  \end{cases}
\]

\begin{definition}[Irreducibility]\label{def:irreducibility}
A class $G$ is $k$-irreducible if, for every $x_1,\ldots,x_k\in\cX$,
\[
  \Ldim\br{G|_{((x_j,\mathrm{SOA}_G(x_j)))_{j=1}^k}}=\Ldim(G).
\]
\end{definition}

\begin{definition}[Decomposition dimension]\label{def:decomposition-dimension}
A decomposition tree for $\varnothing\ne G\subseteq\cH$ is a finite
binary tree whose internal nodes query instances and whose two outgoing
edges are labelled zero and one. Let $G_v$ contain the functions
consistent with the path to node $v$, and write $q_v=\Ldim(G_v)$.
The tree is $(p,d)$-valid if every node satisfies
$\operatorname{depth}(v)\le p(2^{d-q_v+1}-1)$ and every leaf also satisfies
\[
  \operatorname{depth}(v)\le p(2^{d-q_v}-1),
  \qquad G_v\text{ is }p2^{d-q_v}\text{-irreducible}.
\]
Define the \emph{decomposition dimension} by
\[
  \rho_p(G)=\min_{T:\,(p,d)\text{-valid}}
     \max_{v\in\operatorname{Leaves}(T)}\Ldim(G_v).
\]
A tree attaining this minimum is \emph{optimal}.
\end{definition}

\begin{definition}[Essential hypotheses]\label{def:essential-hypotheses}
Let $\operatorname{Opt}_{p,d}(G)$ be the optimal decomposition trees.
The set of \emph{essential hypotheses} is
\[
  C_p(G)=\bigcap_{T\in\operatorname{Opt}_{p,d}(G)}
    \bc{\mathrm{SOA}_{G_v}:
       v\in\operatorname{Leaves}(T),\ \Ldim(G_v)=\rho_p(G)}.
\]
Set $C_p(\varnothing)=\varnothing$.
\end{definition}

We write $\rho_p(G)$ for the decomposition dimension denoted by
$\mathrm{DDim}_{p,d}(G)$ in~\cite{Lyu26}.  The following fact collects
the properties used in \Cref{sec:private-candidates};
see~\cite[Section~4]{Lyu26} for details.

\begin{fact}\label{fact:candidate-properties}
For each nonempty $G\subseteq\cH$ and integer $p\ge1$, the essential
hypotheses $C_p\br{G}$ and decomposition dimension
$\rho_p\br{G}\in\bc{0,\ldots,d}$ satisfy:
\begin{enumerate}[label=(P\arabic*),ref=P\arabic*,leftmargin=*,
                  topsep=3pt,itemsep=0pt,parsep=0pt]
  \item\label{item:candidate-size}
    $\abs{C_p\br{G}}\le p^d2^{d^2}$.

  \item\label{item:candidate-agreement}
    If $h\in C_p\br{G}$, then on any set of at most $p$ instances,
    some $g\in G$ agrees with $h$.

  \item\label{item:candidate-comparison}
    If $\varnothing\ne F\subseteq G$ and $p'\in\bc{p,2p}$, then
    $\rho_{p'}\br{F}\le\rho_p\br{G}$.
    Equality implies $C_{p'}\br{F}\subseteq C_p\br{G}$.

    Moreover, if $\varnothing\ne F\subseteq\bigcap_iG_i$ and
    $\rho_{2p}\br{F}=\rho_p\br{G_i}$ for every $i$, then
    $\bigcap_i C_p\br{G_i}\ne\varnothing$.

  \item\label{item:candidate-rank-zero}
    If $\rho_p\br{G}=0$, then $G$ is finite and $C_p\br{G}=G$.
\end{enumerate}
\end{fact}

\noindent It is important to note that elements of $C_p\br{G}$ need
not belong to $G$. We also use the following private selection
guarantee that follows from the private sparse sampling algorithm
in~\cite[Lemma~3.1]{Lyu26}.

\begin{fact}\label{fact:candidate-selection}
Given $k$ lists $\cC_1,\ldots,\cC_k$, each of size at most $s$,
there is a mechanism $\mathsf{Select}_{a,b}$ that is
$\br{a,b}$-differentially private under replacement of one list.
With probability at least $1-\beta'$, its output satisfies both
of the following:
\begin{enumerate}[leftmargin=*,topsep=3pt,itemsep=0pt,parsep=0pt]
  \item A returned function belongs to at least $3k/5$ lists.
  \item If some function belongs to at least $9k/10$ lists,
    the output is not $\bot$.
\end{enumerate}
It suffices that $k\ge (C/a)\log\br{k(s+1)/(\beta'b)}$.
\end{fact}

The algorithm for constructing the candidate sequence stores samples
in $k$ groups and uses their candidate lists. Let $C_0,C_1$ be sufficiently
large absolute constants. Set
\begin{gamebox}
\[
\begin{aligned}
  n&=\ceil{C_0\br{d+1}^6},\qquad
  s=\br{2^dn}^d2^{d^2},\qquad N=2\br{d+1}s+2,\\
  k&=\ceil{C_1N^2\br{n+\epsilon^{-1}}
       \log\frac{100Nn}{\epsilon\delta\beta}},\qquad K=2kn,\\
  \epsilon_0&=\frac{\epsilon}{8N\br{N+d+1}},\qquad
  \delta_0=\frac{\delta}{8eN\br{N+d+1}}.
\end{aligned}
\]
\end{gamebox}
Here $n$ is the sample size, $s$ bounds each list size, and $K$ is the
mistake-batch size.
Each selection call uses privacy parameters $(\epsilon_0,\delta_0)$.
At level $0\le\ell\le d$, use the error threshold
$a_\ell=1/8-\ell/(32(d+1))$ and list parameter $p_\ell=2^\ell n$.
For a stored sample $A$ and a full batch $B$, define
\[
  \operatorname{err}_A\br{h}
  =\frac1n\sum_{\br{t,x,y}\in A}\mathbf{1}\bc{h\br{x}\ne y},
  \qquad
  \operatorname{err}_B\br{h}
  =\frac1K\sum_{\br{t,x,y}\in B}\mathbf{1}\bc{h\br{x}\ne y}.
\]
The sample error is divided by $n$ even when $A$ has fewer than $n$ records.
Fix a public function $h_0$. Let $\cS_i$ collect the samples
retained by group $i$, and write
\[
  G_i^\ell=\bc{h\in\cH:\operatorname{err}_A(h)\le a_\ell
    \text{ for every }A\in\cS_i}.
\]
These classes are defined using the current sample collections.
The counter $\ell$ records the total number of unsuccessful selections
and is not reset between batches. Each batch uses fresh assignment
and sampling randomness.

\begin{algorithm}[H]
\caption{A private sequence of candidates}
\label{alg:private-candidates}
\small
\algtext*{EndFor}
\algtext*{EndLoop}
\begin{algorithmic}[1]
\Require The complete stream $S$.
\Ensure Functions $h_0,\ldots,h_N$.
\State $\ell\gets0$; $\cS_i\gets\varnothing$ for each $i\in[k]$.
\For{$j=1,\ldots,N$}
  \State $B_j\gets B_K(h_{j-1},S)$.
  \State Assign each record of $B_j$ independently and uniformly to a group,
    obtaining $B_{j,1},\ldots,B_{j,k}$.
  \State For each $i$, append a uniform $\min\bc{n,|B_{j,i}|}$-subset
    of $B_{j,i}$ to $\cS_i$.
  \Loop
    \State $z\gets\mathsf{Select}_{\epsilon_0,\delta_0}
      \br{C_{p_\ell}(G_1^\ell),\ldots,C_{p_\ell}(G_k^\ell)}$.
    \State \textbf{if} $z\ne\bot$ \textbf{then} $h_j\gets z$; exit the loop.
    \State \textbf{if} $\ell=d$ \textbf{then return}
      $(h_0,\ldots,h_{j-1})$, padded with $h_0$ to length $N+1$.
    \State $\ell\gets\ell+1$.
  \EndLoop
\EndFor
\State \Return $h_0,\ldots,h_N$.
\end{algorithmic}
\end{algorithm}

On a realisable stream, the realising
hypothesis belongs to every $G_i^\ell$, so these classes are nonempty.
There are at most $N+d+1$ calls to $\mathsf{Select}$: at most $N$
succeed, while each failure increases $\ell$ or stops the algorithm
when $\ell=d$.

\begin{lemma}\label{lem:candidate-sampling}
With probability at least $1-\beta/2$, there is one set $I\subseteq\bs{k}$,
with $\abs{I}\ge0.99k$, such that the following holds for every full
batch $B_j$. Every group $i\in I$ receives at least $n$ records, and
its stored sample $A$ satisfies
\[
  \sup_{h\in\cH}
  \abs{\operatorname{err}_A\br{h}-\operatorname{err}_{B_j}\br{h}}
  \le \frac1{64\br{d+1}}.
\]
\end{lemma}

\begin{proof}[Proof of Lemma~\ref{lem:candidate-sampling}]
Condition on the past before a full batch $B_j$ is processed.

Call a group unsuccessful for $B_j$ if it receives fewer than $n$
records or its stored sample fails the error bound in the lemma.
A group receives $2n$ records in expectation.
Conditional on receiving at least $n$, its stored sample is a
uniformly chosen $n$-element subset of $B_j$.
Since $\mathrm{VC}\br{\cH}\le d$, a Chernoff bound and uniform
convergence for sampling without replacement bound the probability
that a group is unsuccessful by $1/\br{1000N}$, provided
\[
  n\ge C\br{d+1}^2
    \br{d\log\br{d+1}+\log\br{1000N}}.
\]
Our choice of $n$ satisfies this inequality.

Let $Z_j$ count the unsuccessful groups, so
$\bE[Z_j\mid\text{past}]\le k/(1000N)$.
To bound $Z_j$, use the following equivalent sampling rule. Give
each record an independent uniform random number in $\bs{0,1}$, and retain the
$n$ smallest numbers within each group.
Changing one record's assignment or random number changes $Z_j$
by at most two. Bounded differences over the $K=2kn$ independent
record choices gives
\[
  \bP\bs{Z_j>\frac{k}{100N}\mid\text{past}}
  \le
  \exp\br{-\frac{ck}{N^2n}}
  \le \frac{\beta}{2N},
\]
where the last inequality follows from the choice of $k$.

A union bound over at most $N$ batches shows that, with probability
at least $1-\beta/2$, every full batch has at most $k/\br{100N}$
unsuccessful groups.
Let $I$ contain the groups that are successful in every full batch.
Then $\abs{I}\ge0.99k$.
\end{proof}

\begin{lemma}\label{lem:candidate-selection}
With probability at least $1-\beta/2$, every call to $\mathsf{Select}$
in Algorithm~\ref{alg:private-candidates} satisfies the following:
a returned function belongs to at least $3k/5$ input lists, and if
the output is $\bot$, no function belongs to $9k/10$ input lists.
\end{lemma}

\begin{proof}[Proof of Lemma~\ref{lem:candidate-selection}]
Every input list has size at most $s$, because $\ell\le d$.
Apply Fact~\ref{fact:candidate-selection} to each call with failure
probability $\beta/\br{2\br{N+d+1}}$.
Its required inequality is
\[
  k\ge
  \frac{C}{\epsilon_0}
  \log\frac{2k\br{s+1}\br{N+d+1}}{\beta\delta_0},
\]
which holds for sufficiently large $C_1$.

Apply the guarantee conditional on the preceding calls. A union
bound over at most $N+d+1$ calls proves the lemma.
\end{proof}

\begin{lemma}\label{lem:candidate-list-monotonicity}
On every realisable stream and for every run of
\Cref{alg:private-candidates}, each group's decomposition dimension
is nonincreasing. If it is unchanged between two consecutive list
constructions, the new list is contained in the previous list.
Every list has size at most $s$.
\end{lemma}

\begin{proof}[Proof of \Cref{lem:candidate-list-monotonicity}]
Each update replaces a group's class by a nonempty subclass and
either preserves or doubles $p_\ell$.
The comparison property in \Cref{fact:candidate-properties}
therefore shows that the decomposition dimension cannot increase.
When it is unchanged, the new list is contained in the previous one.
Since $p_\ell\le2^dn$, the size bound in the same fact gives
$|C_{p_\ell}(G_i^\ell)|\le p_\ell^d2^{d^2}\le s$.
\end{proof}

\begin{lemma}\label{lem:candidate-selection-completion}
Fix a realisable stream and random choices for which both guarantees
in \Cref{lem:candidate-sampling,lem:candidate-selection} hold.
If every batch processed by \Cref{alg:private-candidates} is full,
the algorithm makes $N$ successful selections, producing
$h_1,\ldots,h_N$.
\end{lemma}

\begin{proof}[Proof of \Cref{lem:candidate-selection-completion}]
Let $I$ be the set from \Cref{lem:candidate-sampling}.
For $i,i'\in I$ and $\ell<d$, take
$h\in G_i^{\ell+1}$. If $A'$ is a sample stored by group $i'$,
let $A$ be the sample stored by group $i$ from the same batch
$B$. The sampling guarantee gives
\[
\begin{aligned}
  \operatorname{err}_{A'}(h)
  &\le\operatorname{err}_{B}(h)+\frac1{64(d+1)}\le\operatorname{err}_{A}(h)+\frac1{32(d+1)}\le a_{\ell+1}+\frac1{32(d+1)}=a_\ell.
\end{aligned}
\]
Hence $G_i^{\ell+1}\subseteq G_{i'}^\ell$ for every $i,i'\in I$ with
the current stored samples. If selection fails while $\ell<d$, write
\[
  R=\max_{i\in I}\rho_{p_\ell}(G_i^\ell),\qquad
  R'=\max_{i\in I}\rho_{2p_\ell}(G_i^{\ell+1}).
\]
By \Cref{fact:candidate-properties} and the preceding inclusions,
$R'\le\min_{i\in I}\rho_{p_\ell}(G_i^\ell)\le R$.
If $R'=R$, choose a group attaining $R'$. Its new class is
contained in every old class and has the same decomposition
dimension as all of them. By the common-function property in
\Cref{fact:candidate-properties}, some function belongs to every
$C_{p_\ell}(G_i^\ell)$, $i\in I$. It belongs to at least $0.99k$
lists, contradicting \Cref{lem:candidate-selection}. Therefore $R'<R$.

The maximum decomposition dimension over $I$ initially is at
most $d$, and each unsuccessful selection with $\ell<d$
decreases it by at least one. Adding samples cannot increase it
either, by \Cref{lem:candidate-list-monotonicity}. Thus,
if $\ell$ reaches $d$, every group in $I$ has decomposition
dimension zero. Its list then equals its class, by
\Cref{fact:candidate-properties}, and contains the realising
hypothesis. It belongs to at least $0.99k$ lists, so
selection cannot fail when $\ell=d$.

Thus the algorithm completes $N$ successful selections.
\end{proof}

\begin{lemma}\label{lem:candidate-privacy}
Algorithm~\ref{alg:private-candidates} is
$\br{\epsilon,\delta}$-differentially private under replacement
of one record of $S$.
\end{lemma}

\begin{proof}[Proof of Lemma~\ref{lem:candidate-privacy}]
Use the equivalent sampling rule from the proof of
Lemma~\ref{lem:candidate-sampling}.
Independently for each batch number $j$ and original timestamp $t$,
draw a group assignment and a uniform random number.
Fix these choices and average only over the randomness of
$\mathsf{Select}$.

Let $S,S'$ differ at one timestamp.  Compare a selection call on the
two streams after the same preceding selection answers. The preceding
candidates and the count $\ell$ of unsuccessful selections are then
the same. For a fixed candidate, changing one stream record changes
its first-$K$-mistake batch at at most two timestamps: the changed
timestamp itself, and possibly one other timestamp entering or leaving
the batch.

Common records receive the same group assignments and sampling
numbers. Thus at most two groups' samples differ in each batch.
Across at most $N$ batches, the stored samples differ in at most $2N$
groups. By Fact~\ref{fact:candidate-selection} and group privacy
(Fact~\ref{prop:group-privacy}), each selection call is
$(2N\epsilon_0,2N e^{2N\epsilon_0}\delta_0)$-differentially private
with respect to the original stream.  Composing over at most $N+d+1$
calls gives parameters
\[
  \br{
    2N\br{N+d+1}\epsilon_0,\,
    2N\br{N+d+1}e^{2N\epsilon_0}\delta_0
  }.
\]
By the choices of $\epsilon_0,\delta_0$, these are at most
$\br{\epsilon/4,\delta/4}$. Averaging over the assignment and sampling
randomness preserves privacy and the final returned sequence is a
post-processing of the selection answers.
\end{proof}

We restate \Cref{lem:candidate-sequence} with the target parameters
of~\Cref{sec:private-candidates}.

\CandidateSequence*

\begin{proof}[Proof of \Cref{lem:candidate-sequence}]
By \Cref{lem:candidate-sampling,lem:candidate-selection} and a union
bound, both guarantees hold with probability at least $1-\beta$.
Condition on this event and let $I$ be the set in
\Cref{lem:candidate-sampling}. Every selected function belongs to
at least $3k/5-k/100\ge k/2$ lists of groups in $I$.

If every candidate made at least $K$ mistakes, every batch would
be full, and \Cref{lem:candidate-selection-completion} would give
$N$ successful selections. The counting argument in Step~1 of
Section~\ref{sec:private-candidates} then gives
$(N-1)k/2\le k(d+1)s$, contradicting $N=2(d+1)s+2$.
Thus some candidate has fewer than $K$ mistakes. The parameter choices give
\[
  s,N\le2^{\bigO{d^2}},\qquad
  K\le\frac{2^{\bigO{d^2}}}{\epsilon}
       \log\frac{2}{\epsilon\delta\beta}.
\]
Each $h_j$ is computed using only
$B_K(h_0,S),\ldots,B_K(h_{j-1},S)$.
Together with \Cref{lem:candidate-privacy}, substituting
$(\epsilon/4,\delta/8)$ and $\beta=\delta$ proves the lemma.
\end{proof}

\subsection{Using the candidates online}
\label{sec:online-prediction}

For the remainder of this section, we use the notation and parameter
assumptions of Section~\ref{sec:private-candidates}.

The candidates from Step~1 can be computed as their preceding
mistake batches become full. To use them privately, we delay when a
new candidate can affect the predictions.

We use the private delay mechanism $\mathsf{Mirror}$
of~\cite[Lemmas~23--24]{CLNSS24}.
A round is active when its input bit is $1$.

\begin{fact}\label{fact:online-mirror}
For $0<a\le1$, $0<b<1/10$ and an integer $D\ge1$, there is an
$\br{a,b}$-jointly private mechanism $\mathsf{Mirror}\br{a,b,D}$
that receives an activity bit and returns $\bot$ or $\top$ on each
round. For every value of its randomness, it returns $\bot$ on every
inactive round and on the first $D$ active rounds. It returns $\top$
on every active round after the first
$D+Ca^{-2}\log^2\br{2/b}$ active rounds, for an absolute constant $C$.
\end{fact}

\begin{definition}[Replacement mechanism]\label{def:online-replacement}
Let $0<a\le1$, $0<b<1/10$, and let $D$ be an integer with
$D\ge Ca^{-2}\log^2(4/b)$ for a sufficiently large absolute constant
$C$. Given two fixed prediction sequences $U=(u_t)$ and $V=(v_t)$,
run $\mathsf{Mirror}(a/2,b/2,D)$ from
\Cref{fact:online-mirror} on the bits
$\Delta_t=\mathbf1\bc{u_t\ne v_t}$. Write $\mu_t$ for its response
at round $t$.

Let $c$ count the rounds on which $v_t=y_t$ among the first $D$
disagreements, or among all disagreements if there are fewer than $D$.
Independently of Mirror, draw $\xi$ from the Laplace distribution
of scale $4/a$, conditioned on $|\xi|\le D/8$, and set
$\mathrm{acc}=\mathbf1\bc{c+\xi\ge D/2}$. Define
$\mathsf{Rep}(U,V,D,a,b)$ by
\[
  W_t=
  \begin{cases}
    v_t,&\text{if }\mu_t=\top\text{ and }\mathrm{acc}=1,\\
    u_t,&\text{otherwise.}
  \end{cases}
\]
Mirror returns $\bot$ on the first $D$ disagreement rounds, so the
acceptance decision is used only after all labels entering $c$ have
arrived. Thus the replacement is online whenever $U,V$ are.
Both predictors receive the original examples and labels, and neither
receives the replacement's outputs.
\end{definition}

For the parameters in the main proof, choose a sufficiently large
absolute constant $C$ and set
\begin{equation}\label{eq:online-replacement-parameters}
  D_0=\ceil{C\br{
       K+\frac{N^2}{\epsilon^2}\log^2\frac{32N}{\delta}
       }},
  \qquad D_j=9^jD_0,\quad 1\le j\le N,
\end{equation}
and $a=\epsilon/(4N)$, $b=\delta/(8N)$.
Then $D_0>4K$, and every $D\ge D_0$ satisfies the lower bound
in \Cref{def:online-replacement}.
The notation $\mathsf{Rep}(U,V,D)$ in the main text suppresses $a,b$.

\begin{lemma}\label{lem:online-replacement}
For every integer $D\ge D_0$, every fixed pair of prediction sequences
$U,V$, and every value of the replacement mechanism's randomness,
its output $W=\mathsf{Rep}(U,V,D,a,b)$ from
\Cref{def:online-replacement} satisfies:
\begin{enumerate}[label=(\roman*),leftmargin=*,topsep=3pt,itemsep=0pt,parsep=0pt]
  \item If $M(V;S_{1:r})<K$, then
    $M(W;S_{1:r})\le K+2D$.
  \item If $M(U;S_{1:t-1})\le D/4$, then $W_t=U_t$.
\end{enumerate}
\end{lemma}

\begin{proof}
At a disagreement round, exactly one of $U$ and $V$ is correct.
In particular, every round counted by $c$ is a mistake of $U$.

For (i), a mistake on a round where $\Delta_t=0$ is also a mistake
of $V$. The first $2D$ disagreement rounds contribute at most $2D$
additional mistakes. If the prefix contains no further disagreements,
this already proves the bound.

Otherwise, among the first $D$ disagreement rounds, fewer than
$K\le D/4$ are mistakes of $V$. Consequently $c>3D/4$ and
$c+\xi>5D/8>D/2$, so $\mathrm{acc}=1$.
By \Cref{fact:online-mirror}, the choice
$D\ge Ca^{-2}\log^2\br{4/b}$, with a sufficiently large constant,
ensures that $\mu_t=\top$ on every disagreement round after the first
$2D$. The replacement then predicts $v_t$, and any further mistake
is a mistake of $V$. Therefore
$M\br{W;S_{1:r}}\le M\br{V;S_{1:r}}+2D<K+2D$.

For (ii), if $\Delta_t=0$, the two predictors agree and the
replacement returns their common bit. If at most $D$ disagreement
rounds have occurred up to and including round $t$, then
$\mu_t=\bot$, so the replacement returns $u_t$.
Otherwise, the first $D$ disagreement rounds all lie in $S_{1:t-1}$.
Since $U$ has made at most $D/4$ mistakes on that prefix, $c\le D/4$.
The noise satisfies $\abs{\xi}\le D/8$, and hence
$c+\xi\le3D/8<D/2$. Thus $\mathrm{acc}=0$, and again the prediction
is $u_t$.
\end{proof}

Run \Cref{alg:private-candidates} with privacy parameters
$(\epsilon/4,\delta/8)$ and failure probability $\beta=\delta$.
First suppose that all candidates are available and define
\[
  \cA_0=h_0,\qquad
  \cA_j=\mathsf{Rep}(\cA_{j-1},h_j,D_j,a,b),\quad 1\le j\le N.
\]
Each replacement uses independent randomness, also independent of
the candidate procedure. Write $\cA_{j,t}$ for the prediction of
$\cA_j$ at round $t$, and use the same random choices in the online
implementation.

\begin{claim}\label{clm:online-agreement}
Let $0\le m<N$. If $B_K\br{h_m,S_{1:t-1}}$ is not full, then
$\cA_m,\cA_{m+1},\ldots,\cA_N$ agree on every round through $t$.
This holds for every stream and every value of the randomness.
\end{claim}

\begin{proof}
For $0\le j\le N$ and any prefix with $M\br{h_j;S_{1:r}}<K$,
we have
\begin{equation}\label{eq:online-prefix-bound}
  M\br{\cA_j;S_{1:r}}\le K+2D_j.
\end{equation}
For $j=0$ this follows from $\cA_0=h_0$. For $j\ge1$ it is
Lemma~\ref{lem:online-replacement}(i), since $D_j\ge D_0\ge4K$.

Fix $m,t$ as in the claim. Since the batch is not full,
$M\br{h_m;S_{1:s-1}}<K$ for every $s\le t$.
We prove by induction on $j=m+1,\ldots,N$ that $\cA_j$ and
$\cA_{j-1}$ agree at every round $s\le t$.

For $j=m+1$, \eqref{eq:online-prefix-bound} and $4K\le D_0\le D_m$
give $M\br{\cA_m;S_{1:s-1}}\le K+2D_m\le D_{m+1}/4$.
Lemma~\ref{lem:online-replacement}(ii) now shows agreement at round $s$.

For $j>m+1$, the induction hypothesis says that $\cA_{j-1}$ and
$\cA_m$ agree through $t$, so $\cA_{j-1}$ makes at most
$D_{m+1}/4\le D_j/4$ mistakes before round $s$.
Lemma~\ref{lem:online-replacement}(ii) completes the induction.
\end{proof}

We store the observed examples and the predictions of each $\cA_j$.
When a new candidate $h_j$ becomes available, we run its replacement
rule on the earlier examples, using the stored predictions of
$\cA_{j-1}$ and the random choices fixed above. This determines how
$\cA_j$ would have predicted on those rounds and allows it to continue
on the next round. To compute successive candidates, Algorithm~\ref{alg:private-candidates}
keeps its stored samples and its current value of $\ell$.

\begin{algorithm}[H]
\caption{Using the candidates online}
\label{alg:online-candidates}
\small
\algtext*{EndFor}
\algtext*{EndWhile}
\begin{algorithmic}[1]
\State $m\gets0$; $\cA_0\gets h_0$.
\For{$t=1,\ldots,T$}
  \While{$m<N$ and $M(h_m;S_{1:t-1})\ge K$}
    \State Obtain $h_{m+1}$ from Algorithm~\ref{alg:private-candidates}, resuming it if needed.
    \State Run the replacement defining $\cA_{m+1}$ on $S_{1:t-1}$ as described above.
    \State $m\gets m+1$.
  \EndWhile
  \State Receive $x_t$; compute $\cA_{0,t},\ldots,\cA_{m,t}$ in order
    and predict $\cA_{m,t}$.
  \State Receive $y_t$ and update all initialised replacements.
\EndFor
\end{algorithmic}
\end{algorithm}

All batches count mistakes on the entire stored prefix, including
rounds before the candidate was computed.

At the end of the while loop, either $m=N$ or the batch of $h_m$
is not full. Claim~\ref{clm:online-agreement} therefore shows that
the current prediction equals that of $\cA_N$, and that reconstructing
a replacement preserves its earlier predictions. All computations use
only previous labels and the current instance. Denote this online
predictor by $\cA$.

The high-probability mistake bound
\eqref{eq:prediction-high-probability-bound} follows from
Step~2 of ~\Cref{sec:private-candidates}, using the parameter
choice in~\eqref{eq:online-replacement-parameters}.
We now prove the privacy guarantee.

\begin{lemma}\label{lem:online-candidates-privacy}
$\cA$ is
$\br{\epsilon/2,\delta/4}$-jointly differentially private.
\end{lemma}

\begin{proof}
Fix neighbouring streams that differ at round $i$. Reveal all candidate
functions and acceptance bits, and each Mirror response sequence with
coordinate $i$ omitted.

We use the following composition argument because each replacement
receives predictions from the preceding one. Suppose earlier
mechanisms have released some functions or bits and some sequences
of answers with coordinate $i$ omitted. Fix these outputs. A later
mechanism forms its input at round $t$ from the original record at
$t$, the released functions or bits, and the earlier mechanisms'
answers at that same round. Its inputs outside $i$ are therefore
fixed and identical on the two streams. The input at $i$ can have
different conditional distributions. If the later mechanism's
privacy inequality holds for every pair of possible inputs at $i$,
it also holds after averaging over those conditional distributions.
Adaptive composition then adds the privacy parameters. For a
jointly private sequence, omit output $i$. The omitted answer
may enter the later mechanism at round $i$, but it cannot be copied
directly to another round.

For one replacement rule, regard $\br{u_t,v_t,y_t}_{t=1}^T$ as its
input stream. Replacing one record changes the number of labels
favouring the second predictor among the first $D$ disagreements by
at most two. The untruncated Laplace noise of scale $4/a$ therefore
makes the acceptance decision $\br{a/2,0}$-differentially private.
Truncating this noise to $\bs{-D/8,D/8}$ changes its distribution in
total variation by $q=\exp\br{-aD/32}$.
The truncated decision is consequently
$\br{a/2,\br{1+e^{a/2}}q}$-differentially private.
For $D\ge Ca^{-2}\log^2\br{4/b}$ with a sufficiently large constant,
$\br{1+e^{a/2}}q\le b/2$.
Mirror independently contributes $\br{a/2,b/2}$-joint privacy.

We apply this argument to the definition of $\cA_N$ with all
candidates available. First obtain the candidate sequence, which is
$\br{\epsilon/4,\delta/8}$-differentially private by
Lemma~\ref{lem:candidate-privacy}. For each $j=1,\ldots,N$, process
the acceptance test and Mirror instance of its replacement rule.
Their input at round $t$ is
$\br{\cA_{j-1,t},h_j\br{x_t},y_t}$.
Once the candidates and earlier acceptance decisions are fixed,
the first component is determined by $x_t$ and the earlier round-$t$ Mirror
responses. Thus the composition argument applies and each replacement
adds $\br{a,b}$.
The predictions outside $i$ are determined by these releases and the
unchanged instances, so their privacy parameters are
$\br{\epsilon/4,\delta/8}+N\br{a,b}=\br{\epsilon/2,\delta/4}$.
The online predictor $\cA$ has exactly the same predictions as
$\cA_N$ on every stream when they use the same random choices, by
Claim~\ref{clm:online-agreement} and the implementation above. It
therefore has the same privacy guarantee.
\end{proof}

\subsection{An expected mistake bound independent of \texorpdfstring{$T$}{T}}
\label{sec:prediction-expectation}

To remove the $\delta T$ term in $\bE M(\cA;S)\le B+\delta T$, we
combine $\cA$ with an independent subsampled SOA predictor, as in Step~3.

\begin{fact}\label{fact:candidate-subsampled-soa}
Fix $0<\gamma\le1$ and start with an empty retained sample $R$.
On round $t$, predict $\mathrm{SOA}_{\cH|_R}(x_t)$; after observing
$y_t$, retain $\br{x_t,y_t}$ independently with probability $\gamma$.
If no hypothesis is consistent with $R$, use a fixed public function.
This predictor is $\br{0,\gamma}$-jointly differentially private and
has at most $d/\gamma$ expected mistakes on every realisable stream.
\end{fact}

This is the subsampled-SOA argument
of~\cite[discussion following Theorem~1]{DSS24}.

Run the Name and Shame SOA predictor $\cF$ from
\Cref{fact:candidate-subsampled-soa} independently of $\cA$, retaining
each labelled example with probability $\delta/4$. It is
$(0,\delta/4)$-jointly differentially private and satisfies
$\bE M(\cF;S)\le4d/\delta$.

We combine $\cA$ and $\cF$ using Algorithm~\ref{alg:prediction-expectation}, with
$D=\ceil{C\epsilon^{-2}\log(2/\delta)}$ for a sufficiently
large absolute constant $C$. Write
$\cC=\mathsf{Combine}_D(\cA,\cF)$ for the resulting predictor.

\begin{algorithm}[H]
\caption{Combining $\cA$ with the independent predictor $\cF$}
\label{alg:prediction-expectation}
\small
\begin{algorithmic}[1]
\State $c_1,c_2\gets0$.
\For{$t=1,\ldots,T$}
  \State Receive $x_t$ and predictions $u_t,v_t$ from $\cA,\cF$.
  \State $p_t\gets
    \dfrac{(D+c_1)^2}{(D+c_1)^2+(D+c_2)^2}$.
  \State Predict $u_t$ with probability $p_t$, and $v_t$ otherwise.
  \State Receive $y_t$ and supply it to both $\cA$ and $\cF$.
  \State $c_1\gets c_1+\mathbf1\bc{u_t\ne v_t,\ u_t=y_t}$.
  \State $c_2\gets c_2+\mathbf1\bc{u_t\ne v_t,\ v_t=y_t}$.
\EndFor
\end{algorithmic}
\end{algorithm}

Each choice uses a fresh uniform draw, independent across rounds
and of both predictors.
When their predictions agree, the rule returns their common bit.
The same rule defines $\mathsf{Combine}_D(U,V)$ for arbitrary binary
prediction sequences $U=(u_t)$ and $V=(v_t)$, with $u_t,v_t$ supplied
before predicting and $y_t$ afterwards. When these sequences are
generated by predictors, both receive the true labels but never
the comparison rule's outputs.

\begin{lemma}\label{lem:prediction-combination}
For the chosen $D$, $\mathsf{Combine}_D$ is
$(\epsilon/2,\delta/2)$-jointly private on streams of triples
$(u_t,v_t,y_t)$, receiving $u_t,v_t$ before predicting and $y_t$
afterwards.
For every fixed pair of prediction sequences $U,V$ and labels,
\[
  \bE M(\mathsf{Combine}_D(U,V);S)
  \le3\min\bc{M(U;S),M(V;S)}+2D,
\]
where the expectation is over the comparison rule's randomness.
\end{lemma}

\begin{proof}
Fix $U,V$ and the labels, and let $m$ be the total number of mistakes of either
prediction sequence. Rounds on which this sequence is wrong contribute
at most $m$ mistakes. On the $k$-th disagreement round on which it is
correct, its counter is $k-1$ and the other counter is at most $m$.
The probability of a mistake is therefore at most
$\min\bc{1,\br{D+m}^2/k^2}$. Thus
\[
  \bE M(\mathsf{Combine}_D(U,V);S)
  \le m+\sum_{k\ge1}\min\bc{1,\frac{\br{D+m}^2}{k^2}}
  \le3m+2D,
\]
where the last inequality follows by splitting the sum at $D+m$
and using $\sum_{k>D+m}k^{-2}\le1/(D+m)$.
Choosing the sequence with fewer mistakes proves the mistake bound.

For privacy, fix two neighbouring streams of triples that differ at
round $i$. Let $P,Q$ be the laws of their outputs with coordinate $i$
omitted. Changing one triple changes each counter by at most one, and
all later updates agree. Since the counters depend only on the input
triples, the outputs are independent on each fixed stream. Only
disagreement rounds after $i$ contribute to privacy loss.

On such a round, write $a=D+c_1$, $b=D+c_2$, and use \(a',b'\) for the
neighbouring stream. The probabilities of choosing the first predictor
are $p=a^2/(a^2+b^2)$ and $q=\br{a'}^2/(\br{a'}^2+\br{b'}^2)$.  We
have $a,b,a',b'\ge D$ and $\abs{a-a'},\abs{b-b'}\le1$.  Each squared
count, and their sum, changes in logarithm by at most $2/D$. Hence
\(
  \abs{\log\frac pq}\le\frac4D,
  \abs{\log\frac{1-p}{1-q}}\le\frac4D.
\)
Also,
\[
  \frac{\br{p-q}^2}{q\br{1-q}}
  =\br{
    \frac{\br{ab'-ba'}\br{ab'+ba'}}{\br{a^2+b^2}a'b'}
  }^2
  \le\frac{64}{\br{a+b}^2},
\]
using $\abs{ab'-ba'}\le a+b$, $a/a'+b/b'\le4$, and
$a^2+b^2\ge\br{a+b}^2/2$. Let $Z_t$ be the log-likelihood ratio of
this round's output under the first stream. If $n$ disagreements
precede this round, then $a+b=2D+n$. The inequalities $\log x\le x-1$
and $x\br{\log x}^2\le\br{x-1}^2$ for $x>0$ give
\(
  \abs{Z_t}\le\frac4D,
  \bE\bs{Z_t}\le\frac{64}{\br{2D+n}^2},
  \bE\bs{Z_t^2}\le\frac{64}{\br{2D+n}^2}.
\)
Distinct disagreement rounds have distinct values of $n$, and because
$\sum_{n\ge0}\br{2D+n}^{-2}\le1/D$, the full log-likelihood ratio
$L=\sum_t Z_t$ satisfies
\(
  \bE\bs{L}\le\frac{64}{D},
  \Var\br{L}\le\frac{64}{D},
  \abs{Z_t-\bE\bs{Z_t}}\le\frac8D.
\). Bernstein's inequality, with $\lambda=\log\br{2/\delta}$, gives
\[
  \bP\bs{
    L>\frac{64}{D}+\sqrt{\frac{128\lambda}{D}}
      +\frac{16\lambda}{3D}
  }\le\frac{\delta}{2}.
\]
For the chosen $D$, the threshold is at most $\epsilon/2$ when $C$
is sufficiently large. Therefore, every output event $E$ satisfies
$P(E)\le e^{\epsilon/2}Q(E)+\delta/2$.
Exchanging the two streams proves $\br{\epsilon/2,\delta/2}$-joint
privacy.
\end{proof}

The expected mistake bound for $\cC$ was proved in Step~3 of
Section~\ref{sec:private-candidates}. We finish by proving privacy.

\begin{lemma}\label{lem:prediction-expectation}
The predictor $\cC$ in Algorithm~\ref{alg:prediction-expectation} is
$\br{\epsilon,\delta}$-jointly differentially private.
\end{lemma}

\begin{proof}
Fix neighbouring labelled streams differing at round $i$.
By~\Cref{lem:online-candidates-privacy}
and~\Cref{fact:candidate-subsampled-soa}, the two underlying
prediction sequences, with coordinate $i$ removed, together satisfy
$\br{\epsilon/2,\delta/2}$ privacy.  Conditional on these outputs,
every comparison triple outside $i$ is fixed and identical on the
neighbouring streams. The guarantee of
\Cref{lem:prediction-combination} holds for every two choices of the
missing triple, and hence for their conditional distributions.  The
composition argument in the proof of
Lemma~\ref{lem:online-candidates-privacy} therefore adds
$\br{\epsilon/2,\delta/2}$, giving $\br{\epsilon,\delta}$-joint
privacy for the final predictor.
\end{proof}

\par\endgroup

\section*{Acknowledgement}
The authors acknowledge Daniil Dmitriev, Vikrant Singhal, and
Prathamesh Dharangutte for discussions in various stages of this work.
The author claims all responsibility for the content of this paper.

Repeated interactions withAI~(ChatGPT 5.6) along with  prior
incomplete results from the authors were used to derive an initial
weaker version of~\Cref{thm:main} with~\(\sqrt{\log T}\) dependence.
The current results used AI for completing algebraic arguments,
drawing tikz figures, checking for missing definitions, reading and
understanding  prior work, particularly~\cite{Lyu26,CLNSS24}, and for
proof reading proofs.\clearpage
\bibliographystyle{alphaurl}
\bibliography{references}\end{document}